\documentclass{IEEEtran}

\usepackage{microtype}
\usepackage{graphicx}
\usepackage{amsmath, amssymb, amsthm}
\newtheorem{myTheorem}{Theorem}
\newtheorem{myLemma}{Lemma}
\usepackage{epstopdf}
\usepackage{epsfig}
\usepackage{bm}
\newtheorem{myAssumption}{Assumption}
\newtheorem{Definition}{Definition}
\usepackage{multirow}
\usepackage{subfigure}
\usepackage{booktabs} 
\usepackage{algorithm,algorithmic}
\usepackage{wrapfig}

\usepackage[utf8]{inputenc} 
\usepackage[T1]{fontenc}    
\usepackage{nicefrac}       
\usepackage[compress]{cite}

\usepackage{hyperref}

\begin{document}

\title{Embedding Graph Auto-Encoder for Graph Clustering}

\author{Hongyuan Zhang, Rui Zhang$^*$,\IEEEmembership{~Member,~IEEE}, and Xuelong Li,\IEEEmembership{~Fellow,~IEEE}\\

\thanks{$*$ Rui Zhang is the corresponding author.}

\thanks{Hongyuan Zhang, Rui Zhang, and Xuelong Li are with the School of Computer Science and Center for OPTical IMagery Analysis and Learning (OPTIMAL), Northwestern Polytechnical University, Xi'an 710072, Shaanxi, P. R. China.}

\thanks{E-mail: hyzhang98@gmail.com, ruizhang8633@gmail.com, xuelong\_li@nwpu.edu.cn}

}


\maketitle

\begin{abstract}
Graph clustering, aiming to partition nodes of a graph into various groups
via an unsupervised approach, is an attractive topic in recent years. 
To improve the representative ability, several graph auto-encoder (\textit{GAE}) models, 
which are based on semi-supervised graph convolution networks (\textit{GCN}), 
have been developed and they achieve good results compared with traditional 
clustering methods. 
However, all existing methods either fail to utilize the orthogonal 
property of the representations generated by GAE, 
or separate the clustering and the learning of neural networks. 
We first prove that the relaxed $k$-means will obtain an optimal partition 
in the inner-products used space.
Driven by theoretical analysis about relaxed $k$-means, we design a 
specific GAE-based model for graph clustering to be consistent with 
the theory, namely Embedding Graph Auto-Encoder (\textit{EGAE}). 
Meanwhile, the learned representations are well explainable such that 
the representations can be also used for other tasks.
To further induce the neural network to produce deep features that 
are appropriate for the specific clustering model, 
the relaxed $k$-means and GAE are learned simultaneously. Therefore,
the relaxed $k$-means can be equivalently regarded as a decoder that 
attempts to learn representations that can be linearly constructed by 
some centroid vectors.
Accordingly, EGAE consists of one encoder and dual decoders.
Extensive experiments are conducted to prove the superiority of EGAE and 
the corresponding theoretical analyses.
\end{abstract}


\section{Introduction}
Clustering, which plays important roles in plenty of applications, 
is one of the most fundamental topics 
in machine learning \cite{clustering_base,can,sc1,sc2,FMC-MTL}.
One specific task, namely graph clustering or node clustering \cite{gae,MGAE}, 
is to group nodes of a given graph, which is common in citation networks, social networks, 
recommendation systems, \textit{etc}. 
Some conventional methods (\textit{e.g.}, k-means, DBSCAN \cite{dbscan}, AP \cite{AP}, \textit{etc.})
only utilize the features of nodes but ignore the graph structure.
One applicable kind of clustering model is the graph-based model such 
as spectral clustering \cite{sc1,sc2}. 
These methods only employ the graph but neglect the features of nodes.
Although some models \cite{FKSC} use both features and graph 
structure, the capacity of the model limits the performance a lot.

With the rise of deep learning, 
many efforts have been made to promote the capacities of traditional 
clustering models via neural networks \cite{ae,spectral_net,sae,DFKM,DeepCAN}.
Auto-encoder (\textit{AE}) \cite{ae}, as a classical variant of neural 
network for unsupervised learning, is often employed to perform 
clustering \cite{sae,structae}. 
Roughly speaking, they employ a multi-layer neural network, namely encoder, 
to learn non-linear features and reconstruct raw features from the learned 
features via the decoder. 
However, most of them separate clustering from training auto-encoder. 
DEC \cite{dec} embeds a self-supervised clustering model into AE and 
optimizes it by stochastic gradient descent (\textit{SGD}). 
Since SGD is applied for clustering, DEC converges slowly. 
SpectralNet \cite{spectral_net}, which is not based on AE, intends to extend spectral clustering with neural networks. 
However, all these methods fail to utilize the structure information provided by graph type data.

A highly related task of graph clustering is network embedding, 
a fundamental task that aims to learn latent representations 
(namely embedding) for nodes of a graph. 
Network embedding \cite{netgan,graphgan,deep_walk} has been applied in 
diverse applications, such as community networks 
\cite{graph_embedding_1,graph_embedding_2}, bioscience, recommendation systems \cite{deep_walk}, 
\textit{etc}. 
Since the network embedding and node clustering are compatible, it is 
natural to integrate node clustering and network embedding \cite{MGAE,TADW,deep_walk}. 
Specifically, we can utilize the embedding to perform node clustering.
Due to the success of CNN, graph convolution neural networks (\textit{GCN}) 
have been widely studied for network embedding. 
A focusing problem is how to extend the convolution operation into 
irregular data. 
All existing convolution operations can be classified as spectral 
based methods \cite{spectralgcn,chebnet,gcn} 
and spatial based methods \cite{patchysan,spatial1,gat} 
according to \cite{gcn_survey}. 
On the one hand, spectral-based methods are motivated by the convolution 
theorem, the Fourier transformation, and the characteristics of the Laplacian 
operator \cite{graph_op}. 
In spectral-based models, the spectral domain is regarded as the 
frequency domain. 
On the other hand, spatial-based methods  do not transform the domain 
and focus on how to select nodes to perform convolution. 
For example, PATCH-SAN \cite{patchysan} orders other nodes for a node 
and chooses top $k$ neighbors to perform convolution. 
In particular, GCN proposed in \cite{gcn} combines spectral-based models 
and spatial-based methods. 
It employs linear approximation of convolution kernels via Chebyshev 
polynomials \cite{chebnet}. 
Although GCNs are usually employed for semi-supervised learning,
some graph auto-encoders \cite{gae,MGAE,AdaGAE,GALA}, inspired by auto-encoders, 
are proposed for unsupervised representation learning and graph clustering.
However, all these methods often suffer from overfitting and most of them 
overlook the crucial characteristics of the 
generated representations such that some unsuitable clustering methods 
are applied on embeddings. 
Besides, they separate the clustering process from the optimization of GAE.

In this paper, we propose a GAE-based clustering model, 
Embedding Graph Auto-Encoder (\textit{EGAE}), for both graph clustering 
and unsupervised representation learning. 
The main contributions include: 
\begin{itemize}
    \item[1)] We prove that the relaxed $k$-means can obtain the optimal 
                partition with inner-products if some conditions hold. 
    \item[2)] Since the decoder of GAE rebuilds the graph according to inner-products, 
                a specific architecture is designed to cater to the conditions 
                of theoretical analyses. Besides, the learned embedding is well 
                explainable such that EGAE is also a qualified representation 
                learning model.
    \item[3)] Insteading of learning representations and performing clustering separately,
                the two tasks are processed simultaneously. 
                In particular, the part of clustering can be regarded as a 
                decoder as well. Equivalently, EGAE employs two decoders to 
                capture different information.
\end{itemize}

\begin{figure*}[t]
    \centering
    \includegraphics[width=\linewidth]{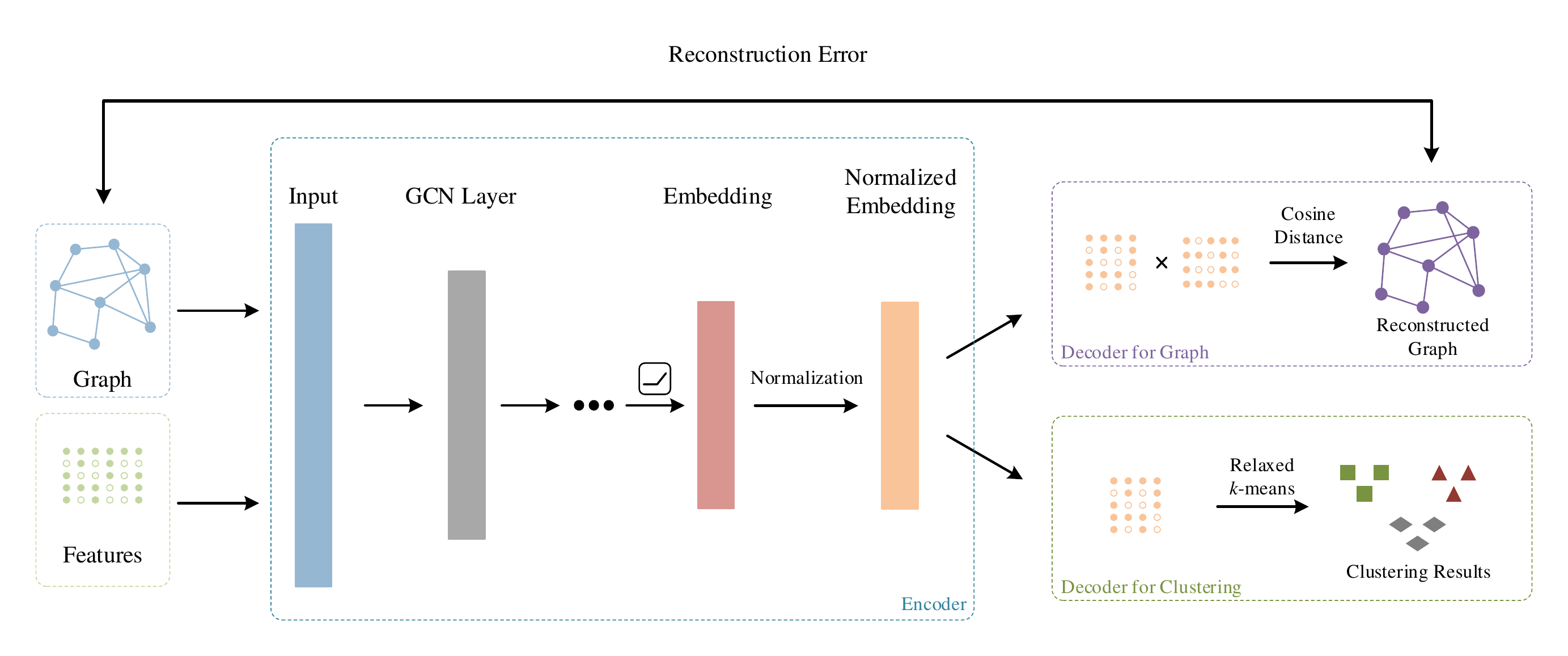}
    \caption{Conceptual illustration of EGAE. Inputs of EGAE consist of two 
            parts, graph and features. 
            After encoding, data is mapped into a latent feature space which 
            employs inner-products as metric. 
            Note that the last layer of encoder utilizes ReLU as 
            activation function to ensure Assumption \ref{assumption_nonnegative_inner_product} 
            holds, and the normalization step is designed due to Assumption 
            \ref{assumption_strong}. 
            Relaxed $k$-means is embedded into GAE to induce it to generate 
            preferable embeddings.
            Note that any complicated layers and mechanisms 
            (\textit{e.g.}, attention, pooling, \textit{etc.}) can be integrated into our framework.}
    \label{figure_architecture}
\end{figure*}

\section{Background}

\subsection{Notations} 
In this paper, all matrices are represented by uppercase words and all vectors 
are denoted by bold lowercase words. 
For a matrix $M$,  ${\rm tr}(M)$ is the trace of $M$ and $M \geq 0$ means 
all elements are non-negative. 
${\rm diag}(\bm m)$ denotes a diagonal matrix whose $(i,i)$-th entry is $m_i$. 
$I$ denotes the identify matrix, $\textbf{1}_n \in \mathbb R^{n}$ denotes 
vector whose elements are all 1. 
If $x$ is a positive scalar, $sign(x) = 1$.
If $x$ is negative, $sign(x) = -1$.
In particular, $sign(0) = 0$.
$\nabla \cdot$ is the gradient operator. 
In general, we use $n$, $d$, and $c$ are used to represent the size of datasets,
dimension of data points, and amount of clusters respectively. 
Given a dataset $\{\bm x_i\}_{i=1}^n$, it can be denoted by 
\begin{equation}
    X = 
    \left [
    \begin{array}{c}
        \bm x_1^T \\
        \bm x_2^T \\
        \vdots \\
        \bm x_n^T
    \end{array}
    \right ] 
    \in \mathbb{R}^{n \times d} .
\end{equation}
The target of clustering is to partition $\{\bm x_i\}_{i=1}^n$ into $c$ groups,
$\{\mathcal{C}_i\}_{i=1}^c$. $|\mathcal{C}_i|$ denotes the amount of samples 
that are assigned to $\mathcal{C}_i$.
We assume that the graph is stored by an adjacency matrix, $A$.



\subsection{Convolution on Graph}
To apply convolution operation on irregular data, we utilize a graph convolution operation that can be explained as both spectral operators \cite{spectralgcn} and spatial operators \cite{patchysan,spatial1}. According to the convolution theorem, the convolution operator can be defined from the frequency domain, which is conventionally named as the spectral domain of graph signals. The adjacency matrix $A$ is employed to represent a graph. Formally, $A_{ij} = 1$ if the $i$-th point is connected with the $j$-th one; $A_{ij} = 0$, otherwise.
$\mathcal L = I - D^{-\frac{1}{2}} A D^{-\frac{1}{2}}$ is the normalized Laplacian matrix where $D$ is diagonal and $D_{ii} = \sum _{j=1}^n A_{ij}$, while $n$ denotes the amount of nodes in the graph. 
Formally, a spatial signal of graph $x \in \mathbb R^n$ can be transformed 
into spectral domain by $U \bm x$ where $\mathcal L = U^T \Lambda U$ 
(eigenvalue decomposition). 
If the convolution kernel $\bm \theta$ is constrained as a function of $\Lambda$, 
a spectral convolution operator can be defined as \cite{gcn_survey}
\begin{equation}
    f(\bm x; \theta) = U^T g(\Lambda; \bm \theta) U \bm x .
\end{equation}
Suppose that $g(\Lambda; \theta)$ is diagonal and can be approximated by Chebyshev polynomials \cite{chebnet}. If the linear approximation is utilized \cite{gcn}, the convolution can be defined as
    $f(\bm x; \bm \theta) = U^T (\theta_0 - \theta_1 \tilde \Lambda) U \bm x = (\theta_0 I - \theta_1 \mathcal{\tilde L}) \bm x$
where $\tilde \Lambda = \frac{2}{\lambda_{max}} \Lambda - I$, $\lambda_{max}$ represents the maximum eigenvalue of $\mathcal L$, and $\mathcal{\tilde L} = U^T \tilde \Lambda U$. To reduce the amount of parameters to learn and simplify the graph convolutional network, suppose that $\theta_0 = - \theta_1$ and $\lambda_{max} \approx 2$. Accordingly, the above equation becomes $f(\bm x; \bm \theta) = \theta_0 (I + D^{-\frac{1}{2}} A D^{-\frac{1}{2}}) \bm x$. Furthermore, we can renormalize the convolutional matrix as
\begin{equation}
    \mathcal {\hat L} = \hat D^{-\frac{1}{2}} (I + A) \hat D^{-\frac{1}{2}} ,
\end{equation}
where $\hat D_{ii} = \sum _{j=1}^n (I + A)_{ij}$. Therefore, the signal processed by convolution can be rewritten as 
    $f(\bm x; \bm \theta) = \theta_0 \mathcal{\hat L} \bm x$.
If the graph signal is multiple-dimensional 
and $d'$ convolution kernels are applied, then we have 
\begin{equation}
    f(X; W) = \mathcal{\hat L} X W ,
\end{equation}
where $W \in \mathbb R^{d \times d'}$ is the parameters to learn. 
From the spatial perspective, $\mathcal{\hat L}$ is the normalized 
Laplacian of $I + A$, which is particularly the adjacency of the 
original graph $A$ with self-loops. 
Particularly, $\mathcal{\hat L} X$ is equivalent to aggregate the 
information from neighbors, 
since $\bm{\hat x}_i = \sum _{j \in \mathcal N_{i}} \mathcal{\hat L}_{ij} \bm x_i$ 
where $\mathcal N_i$ represents the neighbors of node $\bm x_i$ and 
$\bm{\hat x}_i$ is the $i$-th column of $\mathcal{\hat L} X$. 
To further improve the efficiency of convolution, 
some works \cite{MixHopGCN,SGC,AGC} employ the high-order Laplacian to 
equivalently increase the depth of GCNs.

\subsection{Graph Auto-Encoder}
Inspired by the conventional auto-encoders, 
graph auto-encoder \cite{gae} (\textit{GAE}) employs multiple GCN layers to learn 
embeddings of nodes.
Rather than restoring the inputted features from deep representations, 
GAE intends to reconstruct the graph since connections of nodes 
can be regarded as weakly supervised information. 
Specifically, the decoder computes inner-products between any two nodes 
and then maps them into a probability space to model similarities 
via the sigmoid function.
In GAEs \cite{gae,MGAE,AGC,agae}, architectures of the network are often asymmetric 
while a variant \cite{GALA} designs the symmetric architecture via 
graph sharpening. 
For instance, Adversarial Regularized Graph Auto-Encoder \cite{agae} incorporates the 
adversarial learning into GAE to enhance the robustness.
GAE with Adaptive Graph Convolution \cite{AGC} employs the high-order convolution 
operator to promote the capacity of GAE. 

\section{Embedding Graph Auto-Encoder with Clustering}

In this section, we will first discuss the drawbacks of existing models that 
employ graph convolution networks for graph clustering. Then we propose 
\textit{Embedding Graph Auto-Encoder} to address the mentioned problems. 
The architecture of EGAE is illustrated in Fig. \ref{figure_architecture}.

\subsection{Problem Revisit}
In network embedding, divergences between nodes are usually measured by 
inner-product distances (or namely inner-product similarity). 
Formally, 
\begin{Definition} \label{assumption_orth}
    For nodes $i$, $j$, and $k$, node $i$ is more similar with node $j$ 
    than $k$ under inner-products 
    if $\bm x_i^T \bm x_j \geq \bm x_i^T \bm x_k$ where $\bm x_i$ 
    is the representation of node $i$. 
    $\bm x_i^T \bm x_j^T$ is named as the inner-products 
    between $\bm x_i$ and $\bm x_j$.
\end{Definition}
Compared with the traditional auto-encoders, GAE, which is designed as an 
unsupervised model for graph data, attempts to reconstruct the graph matrix 
according to inner-products. 
However, most GAE-based models, which are used for graph clustering, neglect 
that the learned embedding scatters in a inner-product space rather than Euclidean 
space, such that clustering models based on Euclidean distances usually 
give unsatisfactory results. For instance, $k$-means are used for node clustering 
in GAE \cite{gae} and AGAE \cite{agae} but the adjacency matrix is reconstructed 
by inner-products. 
Overall, they pay more attention to how to learn embedding efficiently but 
fail to focus on the clustering task.
Inappropriate clustering methods may provide biased results due to the 
conflicting metrics. 

MGAE \cite{MGAE} and AGC \cite{AGC} utilize the spectral clustering by 
constructing a similarity matrix regarding inner-products. 
Since inner-products may be negative, they simply use absolute values as 
valid similarities, which also results in a problem. Formally, if the 
deep representation learned by GAE is denoted by $Z \in \mathbb{R}^{n \times d'}$,
the similarity matrix is constructed as 
\begin{equation} \label{similarity_improper}
    S = |Z Z^T| .
\end{equation}
Note that $s_{ij} = |\bm z_i^T \bm z_i|$. Provided that $\bm z_i^T \bm z_j < 0$ 
and $\bm z_i^T \bm z_k = 0$, the $i$-th node is more similar to the $k$-th one 
according to the definition of inner-product distance. 
However, $s_{ij} > s_{ik}$ will hold according to Eq. (\ref{similarity_improper}).
The reversed relationships probably mislead the spectral clustering from 
a theoretical aspect.

\subsection{Motivation: An Approach for Clustering in Inner-Product Space} \label{section_motivation}
Before proposing our model formally, we first elaborate the motivation 
theoretically. 
$K$-means, as a fundamental method for clustering, attempts to solve the 
following problem via a greedy method, 
\begin{equation} \label{eq_k_means}
    \begin{split}
        \min \limits_{\bm f_j, g_{ij}} & \sum \limits_{i = 1}^n \sum \limits_{j = 1}^c g_{ij}\|\bm x_i - \bm f_j\|_2^2 \\
        s.t. ~ & g_{ij} \in \{0, 1\}, \sum \limits_{j=1}^c g_{ij} = 1,
    \end{split}
\end{equation}
where $\{\bm f_j\}_{j=1}^c$ denotes centroids of $c$ clusters 
and $g_{ij}$ is the indicator. 
Specifically speaking, 
$g_{ij} = 1$ if the $i$-th point is assigned to the $j$-th cluster.
Otherwise, $g_{ij} = 0$. 
Clearly, $k$-means utilize an implicit assumption that Euclidean distance 
can appropriately depict divergences of data points. 

Let $g_{ij}$ be the $(i, j)$-th entry of matrix $G$ and 
$F = [\bm f_1, \bm f_2, \cdot, \bm f_c] \in \mathbb{R}^{d \times c}$.
Then problem (\ref{eq_k_means}) is equivalent to the following problem, 
\begin{equation}
    \begin{split}
        \min \limits_{F, G} & \|X^T - F G^T\|_F^2, \\
        s.t. ~ & g_{ij} \in \{0, 1\}, G \textbf{1}_c = \textbf{1}_n .
    \end{split}
\end{equation}

The above problem is hard to solve directly due to the discrete constraint 
on $G$. 
Inspired by the spectral clustering, the indicator can be formulated as 
\begin{equation}
    \hat G = G D_G^{-\frac{1}{2}},
\end{equation}
where $D_G$ is a diagonal matrix which satisfies 
$(D_G)_{ii} = \sum _{j=1}^n g_{ji} = |\mathcal{C}_i|$. 
Note that $\hat G$ satisfies that $\hat G^T \hat G = I$.
By substituting $\hat G$ for $G$, the objective of $k$-means can be derived 
as 
\begin{equation} \label{eq_hat_J_c}
    \begin{split}
        \mathcal{\hat J}_c & = \|X^T - F \hat G^T\|_F^2 \\
        & = {\rm tr}(X^T X) - 2 {\rm tr}(X^T \hat G F^T) + {\rm tr}(F \hat G^T \hat G F^T).
    \end{split}
\end{equation}
Take the derivative of $\mathcal{\hat J}_c$ and set it to $0$, 
\begin{equation}
    \nabla_F \mathcal{\hat J}_c = -2 X^T \hat G + 2 F \hat G^T \hat G = 0.
\end{equation}
Accordingly, we have 
\begin{equation} \label{eq_F}
    F = F \hat G^T \hat G = X^T \hat G .
\end{equation}
Furthermore, $\mathcal{J}_c$ can be written as 
\begin{equation}
    \begin{split}
        \mathcal{\hat J}_c = {\rm tr}(X^T X) - {\rm tr}(\hat G^T X X^T \hat G) .
    \end{split}
\end{equation}
Then, problem (\ref{eq_k_means}) is converted into the following problem, 
\begin{equation}
    \begin{split}
        \max \limits_{\hat G} ~ & {\rm tr}(\hat G^T X X^T \hat G) \\
        s.t. ~ & \hat g_{ij} \in \{0, \frac{1}{|\mathcal{C}_j|}\}, \hat G^T \hat G = I.
    \end{split}
\end{equation}

\begin{algorithm}[t]
    \centering
    \caption{Algorithm to optimize problem (\ref{eq_relaxed_k_means}).}
    \label{alg_relaxed_k_means}
    \begin{algorithmic}
        \REQUIRE Data $X$.
        \STATE Calculate $c$ leading left singular vectors, $P$, of $X$.
        \STATE Normalize rows of $P$.
        \STATE Perform $k$-means on normalized $P$.
        \ENSURE Clustering assignments and $P$.
    \end{algorithmic}
\end{algorithm}

Althouth the above problem is still intractable as a result of the discrete 
constraint, it is easy to solve the continuous problem via 
\textit{singular value decomposition (SVD)}, \textit{i.e.},
\begin{equation} \label{eq_relaxed_k_means}
    \begin{split}
        \max \limits_{P} ~ & {\rm tr}(P^T X X^T P) , \\
        s.t. ~ & P^T P = I ,
    \end{split}
\end{equation}
where $P$ is the continuous indicator matrix. 
The algorithm is summarized in Algorithm \ref{alg_relaxed_k_means}.

Although the relaxation seems a common trick from the mathematical perspective, 
Theorem \ref{theorem_ideal} shows us the relaxed $k$-means actually assumes that all 
samples scatter in a inner-product space.
Before providing Theorem \ref{theorem_ideal}, 
two assumptions, the basic of our theoretical analysis, are given 
as follows.
\begin{myAssumption}[Non-Negative Property] \label{assumption_nonnegative_inner_product}
    For any two data points $\bm x_i$ and $\bm x_j$, 
    $\bm x_i^T \bm x_j \geq 0$ always holds.
\end{myAssumption}
Here we discuss the ideal situation: $\bm x_i^T \bm x_j = 0$ 
if and only if $\bm x_i$ and $\bm x_j$ belong to two clusters.
Without loss of generality, $Z Z^T$ can be therefore formulated as 
\begin{equation}
    Q = X X^T = 
    \left [
    \begin{array}{c c c c}
        Q^{(1)} & & & \\
        & Q^{(2)} & & \\
        & & \ddots & \\
        & & & Q^{(c)}
    \end{array}
    \right ]
    ,
\end{equation}
Note that $Q$ can be regarded as a similarity matrix if divergences of 
data points are measured by inner-products. And $Q^{(i)}$ can be thus 
viewed as similarities of samples that belong with the cluster 
$\mathcal{C}_i$. 
\begin{myAssumption} \label{assumption_strong}
    Let $\lambda^{(m)}_i$ be the $i$-th largest eigenvalue of $Q^{(m)}$. 
    For any $a$ and $b$, $\lambda_1^{(a)} > \lambda_2^{(b)}$ always holds.
\end{myAssumption}

\begin{figure}[t]
	\centering
	\includegraphics[width=.8\linewidth]{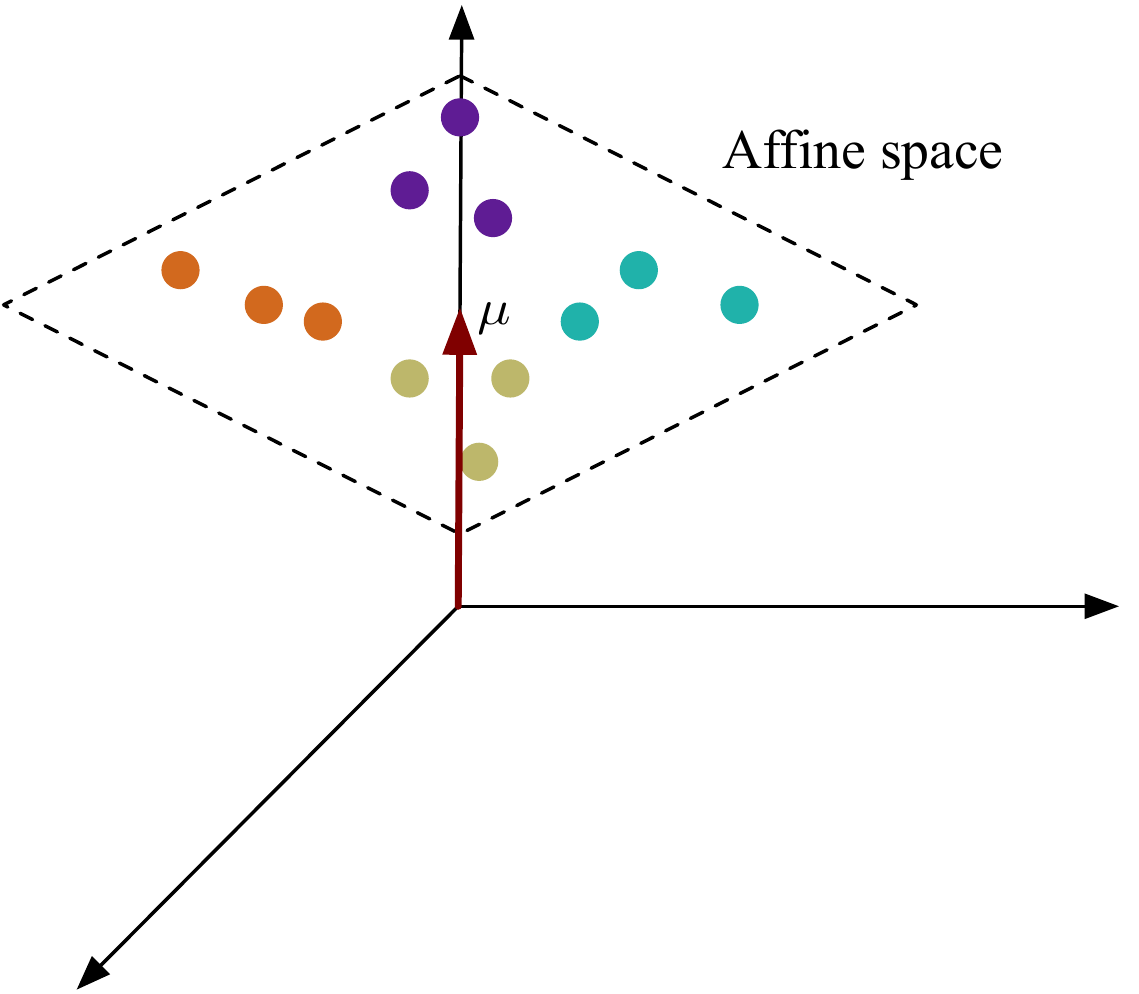}
    \caption{Illustration of Theorem \ref{theorem_connection}. 
                $\mu = \frac{1}{n} \sum _{i=1}^n \bm x_i$ denotes the mean 
                vector of samples.}
	\label{figure_theorem_connection}
\end{figure}

Although Assumption \ref{assumption_strong} seems too strong and impractical, 
we will design a GAE-based model in the next subsection to generate 
qualified embeddings. 
Based on these assumptions, we state Theorem \ref{theorem_ideal} formally.
\begin{myTheorem} \label{theorem_ideal}
    Suppose that Assumption \ref{assumption_nonnegative_inner_product} 
    and \ref{assumption_strong} hold. 
    Provided that $\bm x_i^T \bm x_j = 0$ holds if and only if 
    $\bm x_i$ and $\bm x_j$ belong to two different clusters,
    Algorithm \ref{alg_relaxed_k_means} will give an ideal partition.
\end{myTheorem}

Moreover, the following theorem shows the connection between problem 
(\ref{eq_relaxed_k_means}) and the spectral clustering with normalized cut.
\begin{myTheorem} \label{theorem_connection}
    Problem (\ref{eq_relaxed_k_means}) is equivalent to the spectral 
    clustering with normalized cut if and only if the mean vector 
    $\bm \mu$ is perpendicular to the centralized data. Or equivalently, 
    $\bm \mu$ is perpendicular to the affine space that all data points 
    lie in.
\end{myTheorem}
Fig. \ref{figure_theorem_connection} demonstrates Theorem 
\ref{theorem_connection} vividly.
In the next subsection, we will show details of EGAE based on the above 
theoretical knowledge.

\subsection{Embedding Graph Auto-Encoder}
In practice, Assumption \ref{assumption_nonnegative_inner_product} and 
\ref{assumption_strong} do not hold in most cases, 
especially on the original features. 
However, we can construct a representative model that can map the original 
features into valid representations. 
Based on the theoretical analyses in the previous subsection, we propose 
a theory-driven model, \textit{Embedding Graph Auto-Encoder (EGAE)}, formally.

\begin{figure*}[t]
    \centering
    \subfigure[Raw data]{
        \includegraphics[width=0.21\linewidth]{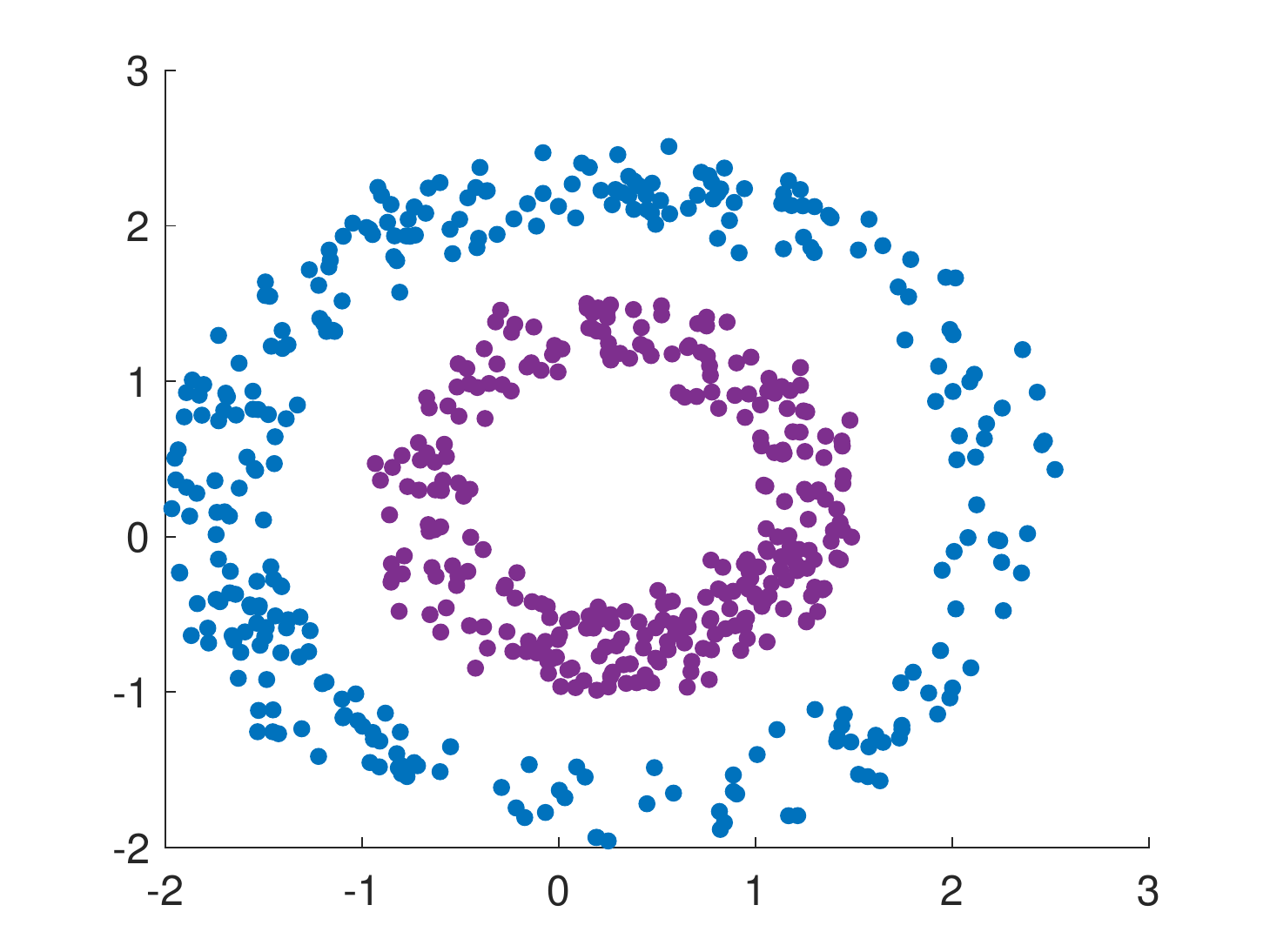}
    }
    \subfigure[After 200 iterations]{
        \includegraphics[width=0.21\linewidth]{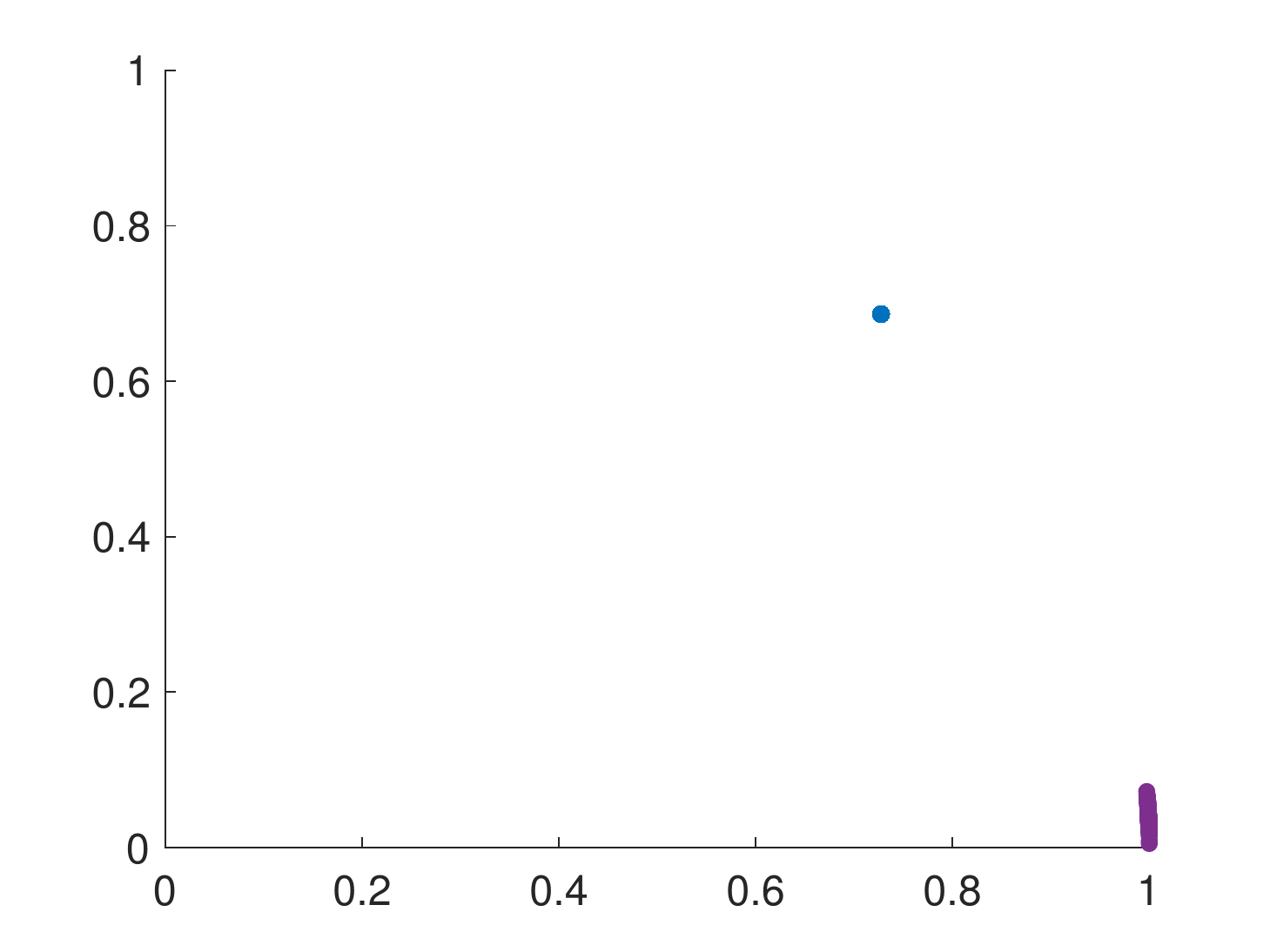}
    }
    \subfigure[After 500 iterations]{
        \includegraphics[width=0.21\linewidth]{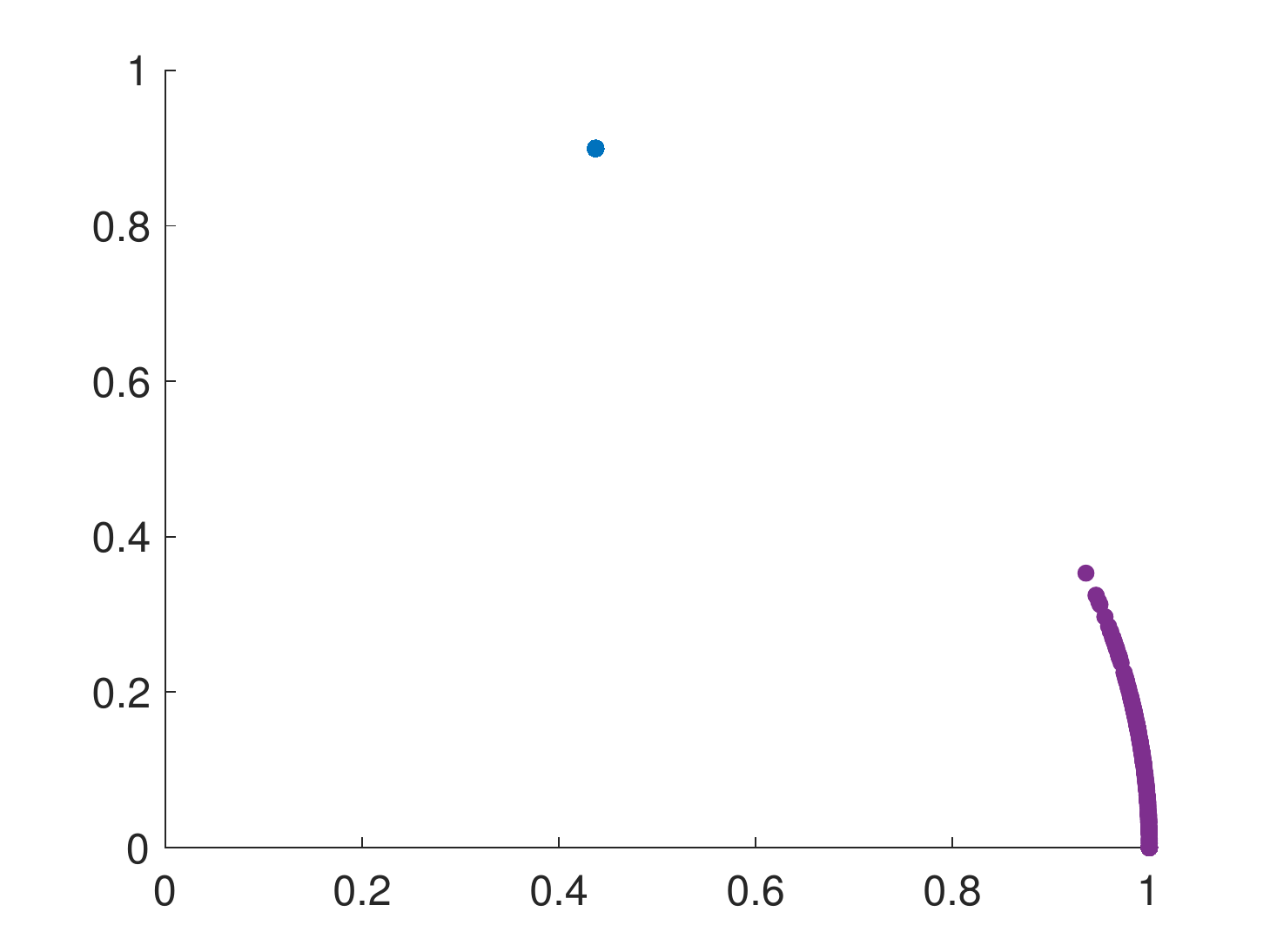}
    }
    \subfigure[After 2000 iterations]{
        \includegraphics[width=0.21\linewidth]{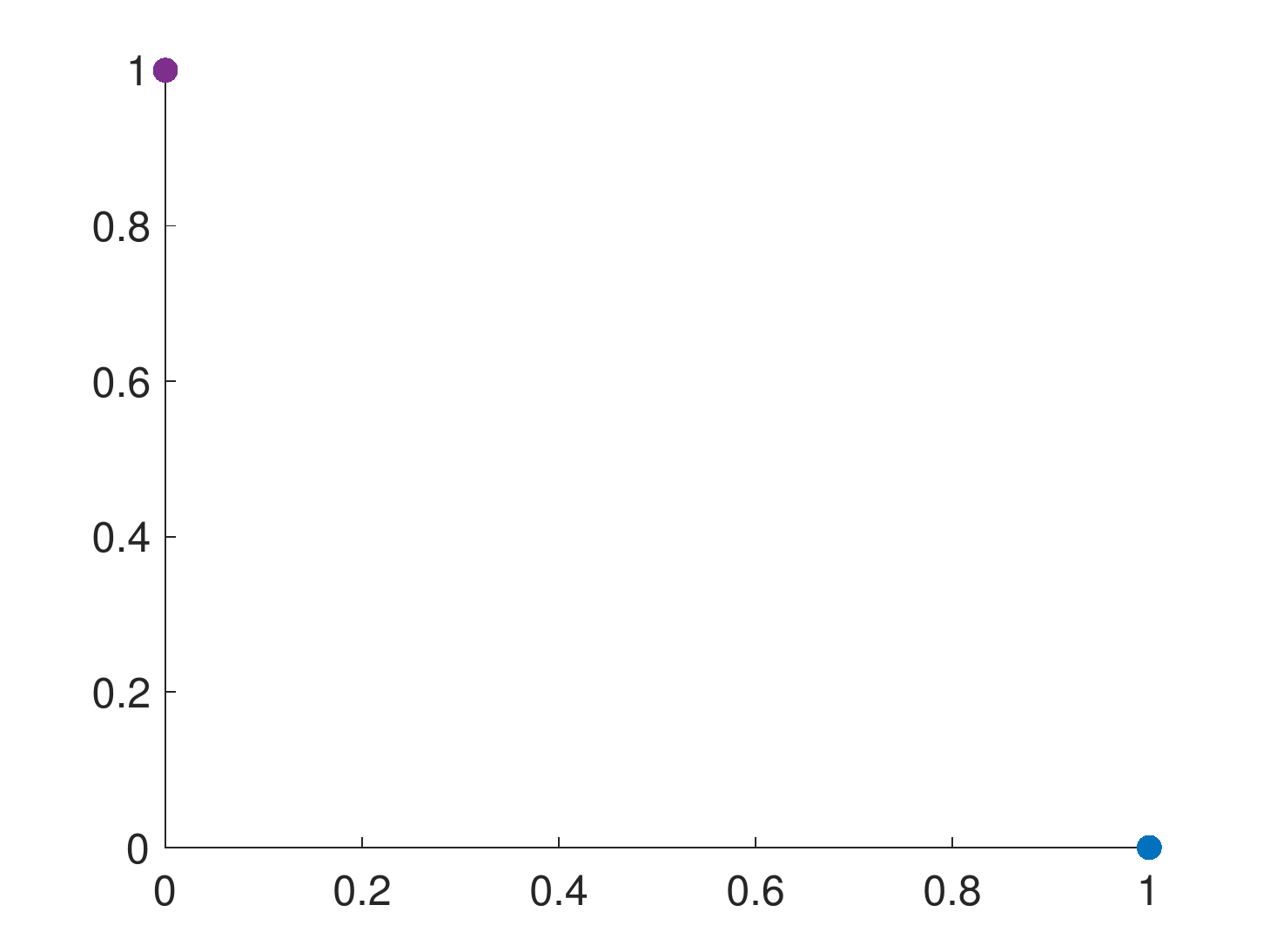}
    }
    
    \caption{Illustration of EGAE with $\alpha = 0$ on the two-rings synthetic 
    dataset. After sufficient training, EGAE projects nodes into orthogonal 
    representations.}
    \label{figure_two_rings}
\end{figure*}

\paragraph{Encoder} 
Encoder aims to learn embedding $Z$ of raw data via multiple graph 
convolution layers. The encoder can be constituted by any valid layers 
and mechanisms such as graph-attention, graph-pooling, \textit{etc}. 
To keep simplicity, only GCN layers are employed in our experiments. 
Specifically speaking, deep representations provided by the $i$-th
layer, denoted by $H_L$, is defined as 
\begin{equation}
    H_i = \varphi_i(\mathcal{\hat L} H_{i-1} W_i) ,
\end{equation}
where $\varphi_i (\cdot)$ is the activation function of the $i$-th layer. 
Let $L$ represent the total amount of layers in the encoder.
Due to Assumption \ref{assumption_nonnegative_inner_product}, 
the learned embedding $Z$ should satisfy that $Z Z^T \geq 0$. 
To satisfy this assumption, 
EGAE employs a simple but effective method by setting
\begin{equation} \label{eq_ReLU}
    \varphi_{L} (\cdot) = \text{ReLU}(\cdot) .
\end{equation}
Accordingly, the generated embeddings satisfy Assumption 
\ref{assumption_nonnegative_inner_product}. 
The key of EGAE is how to design a model to cater to Assumption 
\ref{assumption_strong}.
The following theorem gives us a feasible scheme.
\begin{myTheorem} \label{theorem_eig}
    Suppose that $\forall i, \|\bm z_i\|_2 = 1$ and $\bm z_i^T \bm z_j = 0$
    if and only if the $i$-th data point and the $j$-th one belongs to 
    different clusters. 
    Let $\frac{1}{\epsilon}$ ($\epsilon \geq 1$) denote the lower-bound 
    of non-zero entries, \textit{i.e.}, 
    \begin{equation}
        \left \{ 
        \begin{array}{l l}
            \bm z_i^T \bm z_j = 0, & \text{$\bm z_i$ and $\bm z_j$ belong to different clusters,} \\
            \bm z_i^T \bm z_j \geq \frac{1}{\epsilon}, & \text{otherwise}.
        \end{array}
        \right .
    \end{equation}
    Then, for any valid $a$ and $b$, 
    $\lambda_1^{(a)} > \lambda_2^{(b)}$ holds if 
    \begin{equation} \label{eq_epsilon_constraint}
        \epsilon < \frac{|\mathcal{C}_{min}|}{|\mathcal{C}_{max}| - 2} + 1 ,
    \end{equation}
    where $\mathcal{C}_{min}$ and $\mathcal{C}_{max}$ represent the largest 
    cluster and smallest one, respectively. 
\end{myTheorem}
Before continuing the discussion of the encoder, we provide more insights about 
Theorem \ref{theorem_eig}. Larger $\frac{1}{\epsilon}$ means better 
representations since samples in the same cluster have high similarities. 
When $\epsilon \rightarrow 1$, data points of the same
cluster are projected into the identical representation.
Besides, Eq. \ref{eq_epsilon_constraint} demonstrates that the inequality 
becomes hard to conform to if there is a large gap between the size of the 
largest cluster and smallest one. 
Specifically speaking, when $|\mathcal{C}_{max}| \gg |\mathcal{C}_{min}|$,
$\frac{|\mathcal{C}_{min}|}{|\mathcal{C}_{max}| - 2} + 1 \rightarrow 1$, 
which results in a too strict restriction. 
Provided that $|\mathcal{C}_{min}| = |\mathcal{C}_{max}|$, 
Eq. \ref{eq_epsilon_constraint} becomes 
\begin{equation}
    \epsilon < 2 < \frac{1}{1 - \frac{2 c}{n}} + 1 .
\end{equation}
In other words, the condition of Theorem \ref{theorem_eig} will hold only if 
the similarity of two nodes in the same partition exceeds $0.5$ in the 
category-balanced case.

Based on the above discussion, we can process the output of the $L$-th 
layer via a normalization step. 
Formally, the final representation of embedding is defined as 
\begin{equation} \label{eq_z}
    \bm z_i = \frac{\bm h_i^{(L)}}{\|\bm h_i^{(L)}\|_2},
\end{equation}
where $\bm h_i^{(L)}$ is the $i$-th column of $H_L^T$. 
In summary, the encoder maps nodes into the non-negative part of 
a hypersphere nonlinearly. 

\paragraph{Decoder}
Decoder intends to reconstruct the adjacency $A$ from the embedding and 
the reconstruction error is the objective of GAE. 
The decoder of GAE utilizes a sigmoid function to map 
$(-\infty, +\infty)$ into a probability space.
According to Eq. (\ref{eq_ReLU}) and (\ref{eq_z}), 
it is easy to verify that $\bm z_i^T \bm z_j \in [0, 1]$. 
In other words, the sigmoid function is no more required. 
The reconstructed adjacency can be defined as 
\begin{equation} \label{eq_reconstruction_A}
    \hat A = Z Z^T .
\end{equation}
The objective to reconstruct the adjacency can be formulated as 
\begin{equation}
    \mathcal{J}_r = KL(A || \hat A) .
\end{equation}

\paragraph{Embed Clustering Model as Another Decoder}
Instead of separating clustering and embedding learning, EGAE 
attempts to obtain appropriate representations for the specific model 
defined in problem (\ref{eq_relaxed_k_means}) via optimizing 
it and the network simultaneously. 
Through replacing $\hat G$ in $\mathcal{\hat J}_c$ (defined in 
Eq. (\ref{eq_hat_J_c})) by $P$, the objective of relaxed $k$-means is 
formulated as 
\begin{equation}
    \mathcal{J}_c = {\rm tr}(Z Z^T) - {\rm tr}(P^T Z Z^T P) ,
\end{equation}
and the loss of EGAE is defined as 
\begin{equation} \label{obj}
    \min \limits_{W_i, P^T P = I} \mathcal{J} = \mathcal{J}_r + \alpha \mathcal{J}_c, 
\end{equation}
where $\alpha$ denotes a tradeoff hyper-parameter.
It should be emphasized that $\mathcal{J}_c$ can be employed as an 
individual decoder. 
Since one core idea of unsupervised neural networks is to define 
a loss via reconstruction for training, $\mathcal{J}_c$ provides 
a novel approach to train neural networks unsupervisedly. 
On the one hand, $\mathcal{J}_c$ induces 
the network to produce representations that are preferable for 
the relaxed $k$-means. 
On the other hand, with fixed $P$ and $F$, $\mathcal{J}_c$ aims to 
produce embedding to approach $F P^T$. 
Therefore, it can be viewed as a decoder as well.
Overall, EGAE can be also viewed as a variant of GAE with dual decoders.

Due to the constrained problem in Eq. (\ref{obj}), the simple gradient 
descent can not be applied directly.  
As only $P$ has a constraint,
the alternative method, which consists of gradient descent for
$\{W_i\}_{i=1}^L$ and closed-form solution for $P$, is employed.
Details of the optimization are summarized in Algorithm \ref{alg}.

\textbf{Remark 1}: 
As we all know, one severe problem that traditional auto-encoders suffer 
from is that the representations may be inappropriate for specific tasks 
such as clustering, classification, \text{etc}.
A reason is the uncertain metric on representations. 
In other words, we cannot obtain representations under a specific distance 
metric with the symmetric architecture, which is widely used in auto-encoders.
Compared with AE, 
GAE utilizes the inner-products of deep features to 
reconstruct the graph via an asymmetric structure. 
As we focus on graph data in this paper, GAE is a natural choice to fit 
our theoretical analyses.

\textbf{Remark 2}: Another merit of the normalization defined in 
Eq. (\ref{eq_z}) is to unify the inner-products and Euclidean distance. 
More formally,
\begin{equation}
    \|\bm z_i - \bm z_j \|_2^2 = \|\bm z_i\|_2^2 + \|\bm z_j\|_2^2 - 2 \bm z_i^T \bm z_j = 2 - 2 \bm z_i^T \bm z_j.
\end{equation}
Therefore, the Euclidean distance of two nodes decreases monotonously with 
the increasing of inner-products. 
This property demonstrates the favorable explainability of representations 
generated by EGAE. 
Besides node clustering, the embedding is thereby suitable for other tasks 
such as classification, link prediction, \textit{etc}.

\begin{algorithm}[t]
    \centering
    \caption{Algorithm to optimize problem (\ref{obj}).}
    \label{alg}
    \begin{algorithmic}
        \REQUIRE Tradeoff parameters $\alpha$ and the adjacency matrix $A$. 
        \REPEAT 
            \REPEAT 
                \STATE Calculate the gradients of Eq. (\ref{obj}) w.r.t. $\{W_i\}_{i=1}^l$.
                \STATE Update $\{W_i\}_{i=1}^L$ by the gradient descent.
            \UNTIL{convergence or exceeding maximum iterations.}
            \STATE Compute $P$ and obtain partition via Algorithm \ref{alg_relaxed_k_means}. 
        \UNTIL{convergence or exceeding maximum outer-iterations.}
        \ENSURE Clustering assignments, indicator matrix $G$ and parameters $\{W_i\}_{i=1}^l$.
    \end{algorithmic}
\end{algorithm}

\subsection{Intuitive Illustration}
To understand the motivation of EGAE, we give a visual 
interpretation via experiment on a toy dataset. 
Roughly speaking, EGAE can map data into a hypersphere non-linearly.
Similarities of data points are measured by inner-products 
in the hypersphere. 
Fig. \ref{figure_two_rings} illustrates the effect of EGAE 
on 2-rings data. 
In the experiment, we employ a sampling method to generate the graph.
Specifically, we connect two nodes that belong to an identical cluster 
with 90\% probability, and there is no link between two different clusters.
From the figure, we realize that with the number of iterations 
increasing, 1) connected samples become more cohesive; 
2) samples of different clusters go as orthogonal as possible. 
In other words, we can obtain orthogonal embedding when the prior 
adjacency contains sufficient information.

\begin{figure*}
    \subfigure[Raw features]{
        \includegraphics[width=.3\linewidth]{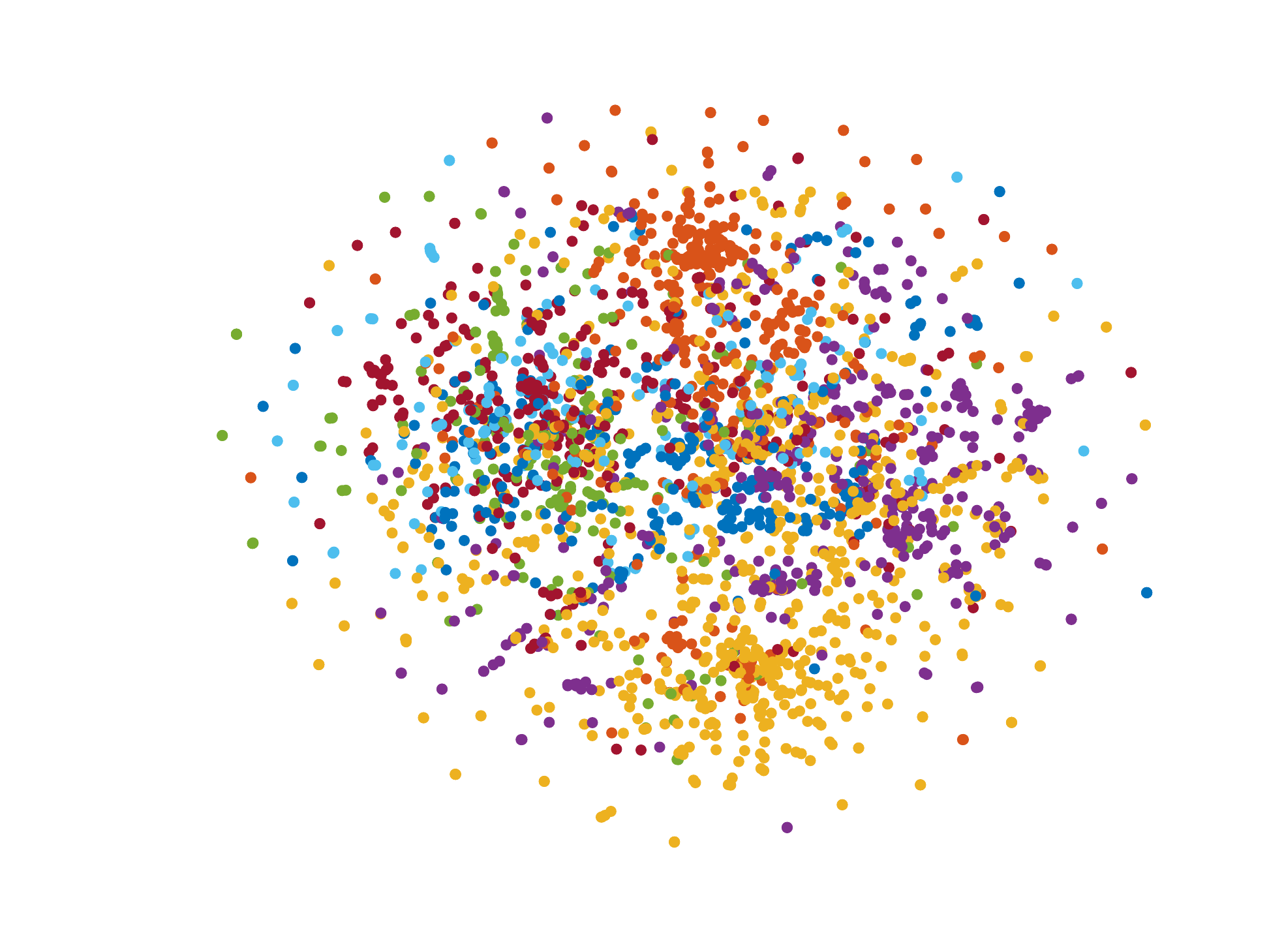}
    }
    \subfigure[Embedding of EGAE ($\alpha = 0$)]{
        \includegraphics[width=.3\linewidth]{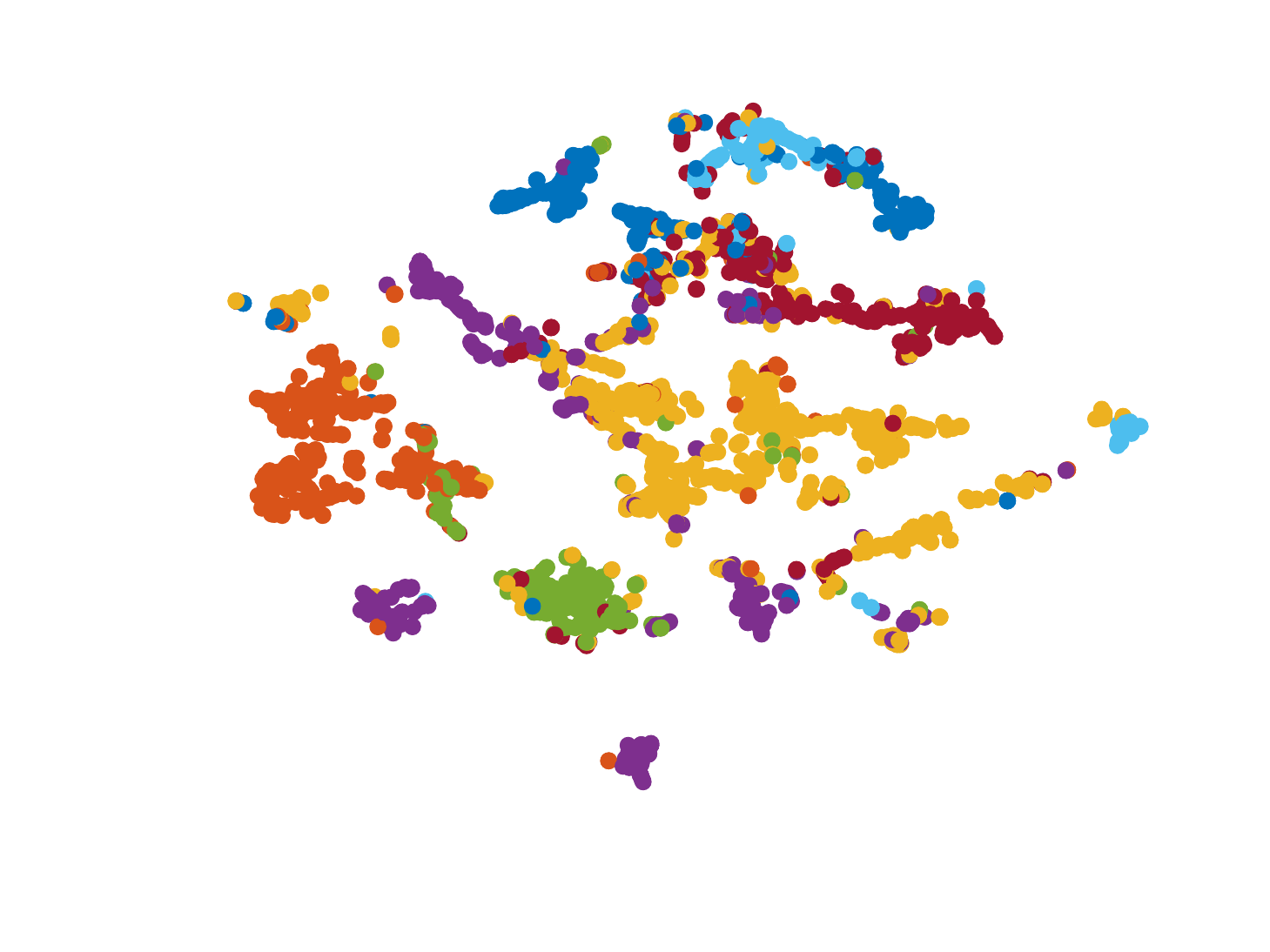}
    }
    \subfigure[Embedding of EGAE]{
        \includegraphics[width=.3\linewidth]{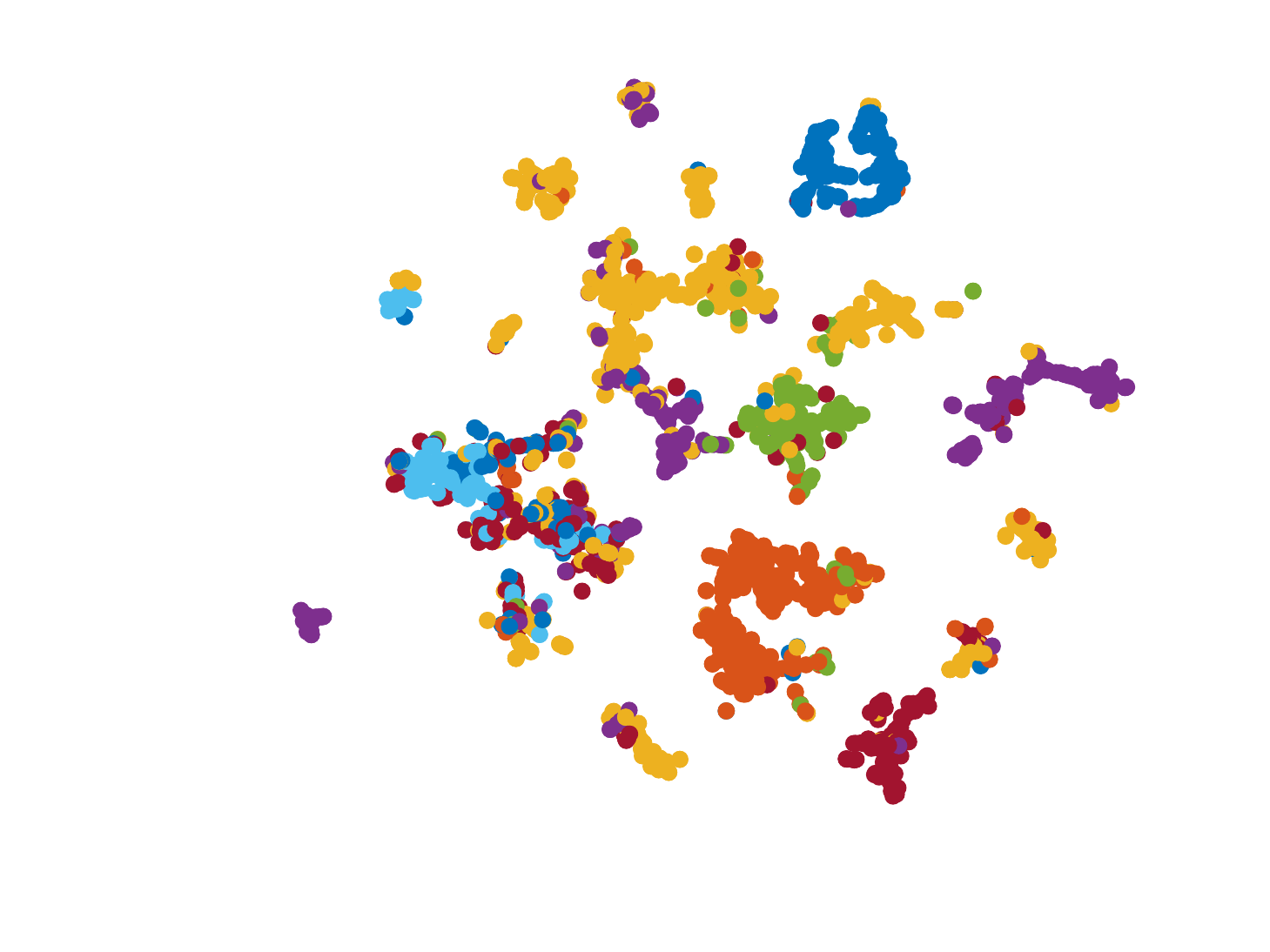}
    }

    \caption{t-SNE visualization of EGAE on Cora.}
    \label{figure_illustration_cora}
\end{figure*}

\subsection{Extension: EGAE with Adjacency Sharing} 

As is shown in Eq. (\ref{eq_reconstruction_A}), 
$Z Z^T$ is the reconstructed adjacency. 
Meanwhile, $Z Z^T$ can be viewed as a similarity matrix in relaxd $k$-means 
according to Eq. (\ref{eq_relaxed_k_means}). 
If the quality of reconstruction is poor, then the performance of relaxed 
$k$-means may be affected severely. 
To alleviate the bias caused by incomplete training, a considerable 
manner is to employ the prior information provided by the graph to 
rectify the clustering result. 

A substantial clustering model that can explore clustering information of 
graphs is the spectral clustering. 
Let $L = D - A$ be the unnormalized Laplacian matrix. 
Then the rectified model that incorporates the prior graph information 
can be defined as 
\begin{equation}
    \label{eq_joint_clustering}
    \min \limits_{F, P^T P = I} \|X^T - F P^T\|_F^2 + \beta {\rm tr}(P^T L P) ,
\end{equation}
where $\beta$ is the tradeoff coefficient to constrol the importance of 
the prior knowledge. Note that the latter term is the spectral clustering 
with ratio cut. To solve the above problem, it is equivalent to 
\begin{equation}
    \min \limits_{P^T P = I} {\rm tr}(P^T (\beta L - X X^T) P) .
\end{equation}
Therefore, the optimal solution of problem (\ref{eq_joint_clustering}) is 
given by eigenvalue decomposition of $XX^T + \beta L$.
Accordingly, the loss of this extension of EGAE is formulated as 
\begin{equation}
    \min \limits_{W_i, P^T P = I} \mathcal{J}_r 
    + \alpha (\|Z^T - F P^T\|_F^2 + \beta {\rm tr}(P^T L P)) .
\end{equation}
In particular, the adjacency matrix is shared by both GAE and the clustering 
part. This mechanism helps EGAE to avoid overfitting by the given graph 
structure. 
In this paper, we only conduct experiments on EGAE without this extension 
since it will slow down the optimization. The details can be found in 
the next subsection.

\subsection{Computational Complexity}
Let $n_e$ and $d_i$ be number of edges in the graph and dimension of 
the $i$-th layer's output, respectively. 
To keep notations uncluttered, 
$d_0 = d$ is used as the dimension of raw features.
Since computation of gradient for each node needs all connected 
samples in each propagation, 
every step to calculate the gradient of $\mathcal{J}_r$ w.r.t. 
$\{W_i\}_{i=1}^L$ requires 
$O(n_e \sum_{i=0}^k (d_i d_{i+1}))$ time. 
To compute $\nabla_{W_i} \mathcal{J}_c$, extra $O(n d_L c)$ time is needed.
The optimization of clustering model needs $O(n d_L c)$ to compute 
the $c$ leading singular of $Z$. 
It should be pointed out that if the extended model is employed, 
then the optimization of clustering model requires $O(n^2 c)$ at least 
which extremely slows down the optimization.

Let $T_{i}$ be the inner-iterations to update parameters of the network 
with fixed $P$ and $T_{o}$ be the outer-iterations.
Then, the time complexity of EGAE is 
$O( T_o(T_i (n d_L c + n_e \sum_{i=0}^k (d_i d_{i+1})) + n d_L c))$. 
Apparently, the embedded clustering model does not increase the complexity 
of GAE as $c$ is usually small compared with $d_i$.


\section{Proofs} \label{section_analysis}

In this section, we will provide proofs of the mentioned theorem 
respectively.

\subsection{Proof of Theorem \ref{theorem_ideal}}
To prove Theorem \ref{theorem_ideal}, the following lemma is required.
\begin{myLemma} \label{lemma}
    For any positive and symmetric matrix $K$, the most principle component 
    $\bm \gamma$ satisfies that all elements are not zero and have the 
    same sign. More formally, 
    \begin{equation}
        \forall i, j, sign(\gamma_i) = sign(\gamma_j)
    \end{equation} 
\end{myLemma}

\begin{proof}
    According to the definition, we have 
    \begin{equation}
        \left \{
        \begin{array}{l}
            \bm \gamma = \arg \max \limits_{\bm x} \frac{\bm x^T K \bm x}{\bm x^T \bm x} \\
            \lambda_{max} = \max \limits_{\bm x} \frac{\bm x^T K \bm x}{\bm x^T \bm x}
        \end{array}
        \right.
    \end{equation}
    
    Suppose that $\exists i, j, sign(\gamma_i) \neq sign(\gamma_j)$. 
    We can construct $\bm {\hat \gamma} = |\bm \gamma|$. 
    Note that 
    \begin{equation}
        \bm x^T K \bm x = \sum \limits_{i, j} k_{ij} x_i x_j
    \end{equation}
    As $K > 0$, we have
    \begin{equation}
        \bm{\hat \gamma}^T K \bm{\hat \gamma} = \sum \limits_{i, j} k_{ij} |\gamma_i \gamma_j| > \sum \limits_{i, j} k_{ij} \gamma_i \gamma_j = \bm{\gamma}^T K \bm{\gamma}
    \end{equation}
    Moreover, $\bm{\hat \gamma}^T \bm{\hat \gamma} = \bm \gamma^T \bm \gamma$. Hence, we have
    \begin{equation}
        \frac{\bm{\hat \gamma}^T K \bm{\hat \gamma}}{\bm{\hat \gamma}^T \bm{\hat \gamma}} = \frac{\bm{\hat \gamma}^T K \bm{\hat \gamma}}{\bm{\gamma}^T \bm{\gamma}} > \frac{\bm{\gamma}^T K \bm{\gamma}}{\bm{\gamma}^T \bm{\gamma}}
    \end{equation}
    which leads to a contradiction. Therefore, $\bm \gamma \geq 0$ or $\bm \gamma \leq 0$. 
    
    Due to $K > 0$, ${\rm tr}(K) = \sum \limits_{i} k_{ii} > 0$. 
    Clearly, we have $\lambda_{max} > 0$. If $\gamma_k$ = 0, then we have
    \begin{equation}
        \bm \eta =  K \bm \gamma = \sum \limits_{i\neq k} \bm k_i \gamma_i
    \end{equation}
    Note that 
    \begin{equation}
        K \bm \gamma = \lambda_{max} \bm \gamma
    \end{equation}
    It is not hard to verify that $\eta_k = \lambda_{max} \gamma_k$ 
    if and only if $\bm \gamma = 0$, 
    which leads to a contradiction. 
    Therefore, $\bm \gamma < 0$ or $\bm \gamma > 0$.
    
    Hence, the lemma is proved.
\end{proof}

Here, we give the complete proof of Theorem \ref{theorem_ideal} based on 
the above lemma.

\begin{proof} [Proof of Theorem \ref{theorem_ideal}]
    If samples from different clusters are orthogonal and inner-products 
    of samples from the same cluster are positive, then we have 
\begin{equation}
    Q_z = Z Z^T = 
    \left [
    \begin{array}{l l l l}
        Q_z^{(1)} & & & \\
        & Q_z^{(2)} & & \\
        & & \ddots & \\
        & & & Q_z^{(c)}
    \end{array}
    \right ]
\end{equation}
where $Q_z^{(i)} = Z_i Z_i^T$ and $Z_i \in \mathbb R^{|\mathcal{C}_i| \times d }$ 
consists of samples from the $i$-th cluster. 
According to our assumption, we have $Q_z^{(i)} > 0$. 
With the help of Lemma \ref{lemma}, there exists $\bm \gamma_i$ 
which satisfies that 
\begin{equation}
    \left \{
    \begin{array}{l}
        \bm \gamma_i > 0 \\
        Q_z^{(i)} \bm \gamma_i = \lambda_{max}^{(i)} \bm \gamma_i
    \end{array}
    \right.
\end{equation}
Furthermore, 
\begin{equation}
    Q_z
    \left [
    \begin{array}{c}
        0 \\
        \bm \gamma_i \\
        0
    \end{array}
    \right ]
    = \lambda_{max}^{(i)} 
    \left [
    \begin{array}{c}
        0 \\
        \bm \gamma_i \\
        0
    \end{array}
    \right ] .
\end{equation}
Therefore, a valid relaxed $G_0$ can be given as 
\begin{equation}
    P_0 = 
    \left [
    \begin{array}{l l l l} 
        \frac{\bm \gamma_1}{\|\bm \gamma_1\|_2} & & & \\
        & \frac{\bm \gamma_2}{\|\bm \gamma_2\|_2} & & \\
        & & \ddots & \\
        & & & \frac{\bm \gamma_c}{\|\bm \gamma_c\|_2}
    \end{array}
    \right ] .
\end{equation}

However, given a orthogonal matrix $R$, 
any $P$ which satisfies $P = P_0 R$ is a valid solution. 
If we normalize rows of $P$ as $\hat P$, then we have 
\begin{equation}
    \hat P = 
    \left [
    \begin{array}{c}
        \pm \textbf{1}_{|\mathcal{C}_1|} \bm r^1 \\
        \pm \textbf{1}_{|\mathcal{C}_2|} \bm r^2 \\
        \vdots \\
        \pm \textbf{1}_{|\mathcal{C}_c|} \bm r^c
    \end{array}
    \right ] ,
\end{equation}
where $\bm r^i$ represents the $i$-th row vector of $R$.

Accordingly, it will obtain ideal partition if the $k$-means is performed 
on rows of $\hat P$.
\end{proof}

\begin{figure*}[t]
    \subfigure[Raw features]{
        \includegraphics[width=.3\linewidth]{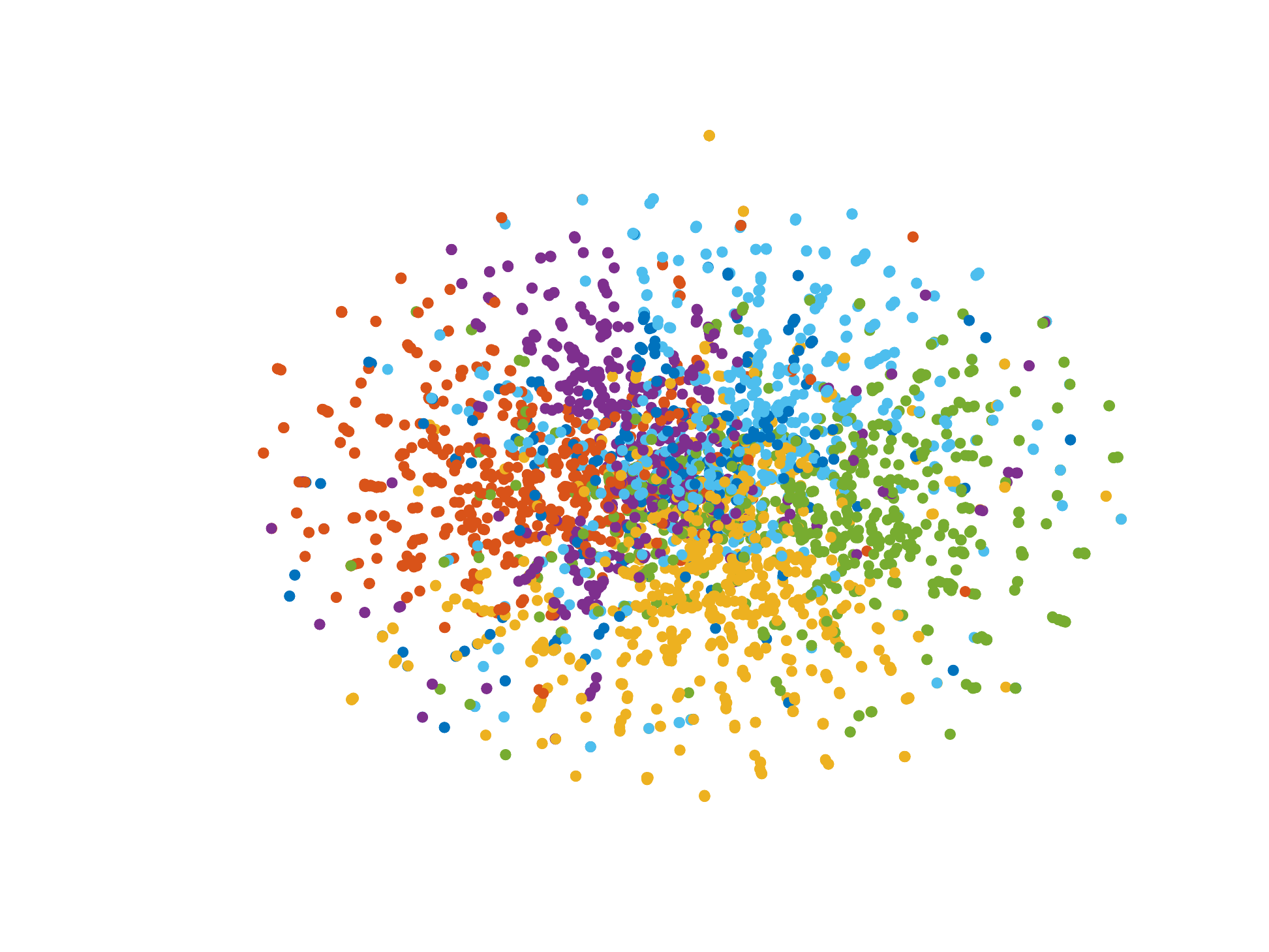}
    }
    \subfigure[Embedding of EGAE ($\alpha=0$)]{
        \includegraphics[width=.3\linewidth]{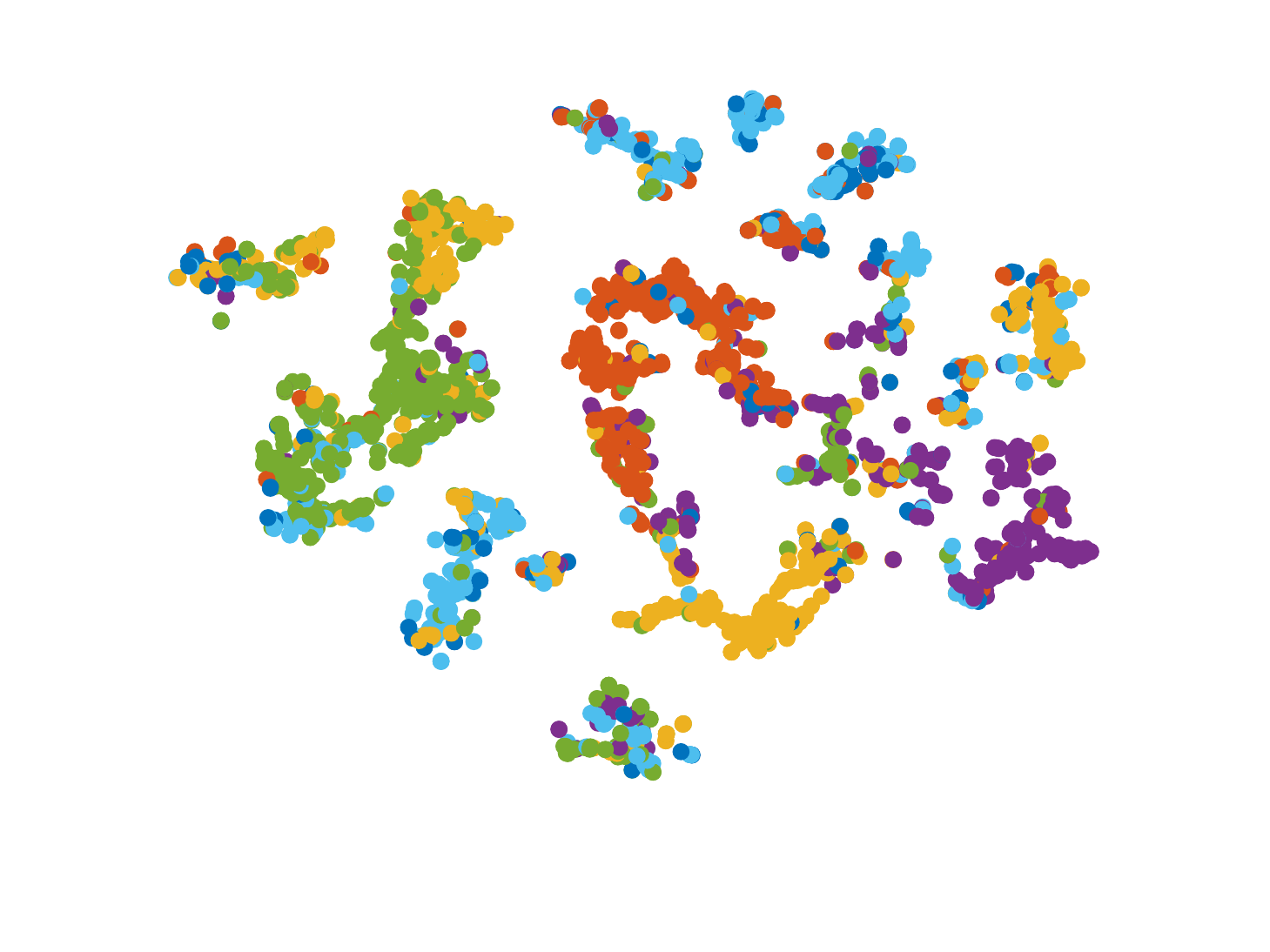}
    }
    \subfigure[Embedding of EGAE]{
        \includegraphics[width=.3\linewidth]{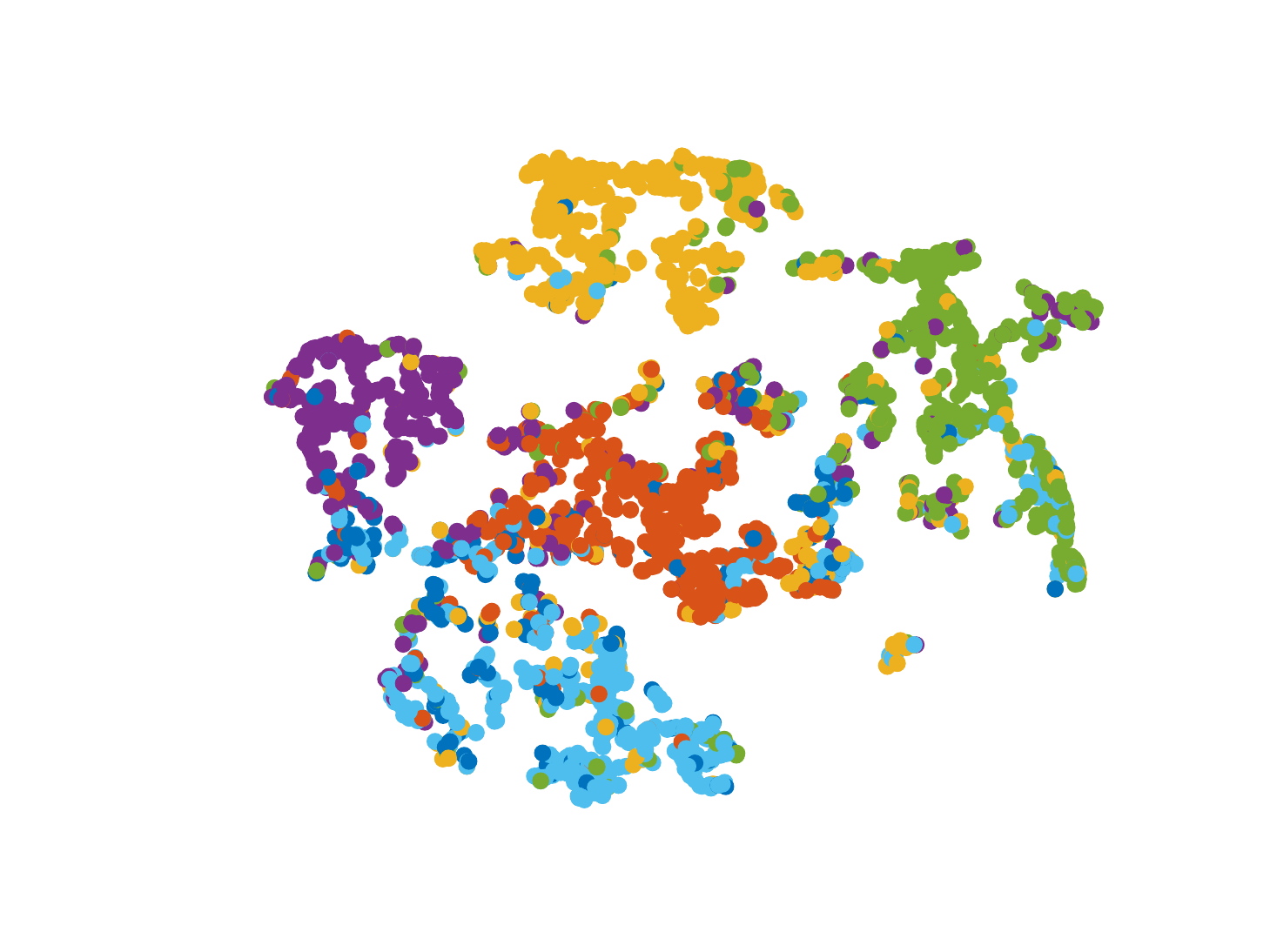}
    }

    \caption{t-SNE visualization of EGAE on Citeseer.}
    \label{figure_illustration_citeseer}
\end{figure*}

\subsection{Proof of Theorem \ref{theorem_connection}}
\begin{proof} 
    The objective function of normalized cut spectral clustering is given as
    \begin{equation}
        \label{objective_norm_cut_sc}
        \begin{split}
            & \min \limits_{P^T P = I} {\rm tr}(P^T \mathcal L P) \\
            \Rightarrow & \min \limits_{P^T P = I} {\rm tr}(I - P^T D^{-\frac{1}{2}} A D^{-\frac{1}{2}} P) \\
            \Rightarrow & \max \limits_{P^T P = I} {\rm tr}(P^T D^{-\frac{1}{2}} A D^{-\frac{1}{2}} P)
        \end{split}
    \end{equation}
    Note that $D = {\rm diag}(Z Z^T \textbf{1}) = n \times {\rm diag}(Z \bm \mu)$ where $\bm \mu = \frac{1}{n} \sum _{i=1}^n \bm z_i$. Let $H = I_n - \frac{1}{n} \textbf{1}_n \textbf{1}_n^T$ be a centralized matrix. Then the centralized data can be represented as $\hat Z = H Z = Z - \textbf{1}_n \bm \mu^T$.
    
    On the one hand, if $\bm \mu$ is perpendicular to the centralized data, then we have 
    \begin{equation}
        \hat Z \bm \mu = Z \bm \mu - \textbf{1}_n \bm \mu^T \bm \mu = Z\bm \mu - \|\bm \mu \|_2^2 \textbf{1}_n = 0
    \end{equation}
    which means 
    \begin{equation}
        D = n \times {\rm diag}(\|\bm \mu \|_2^2 \textbf{1}_n) = n\|\bm \mu\|_2^2 I_n
    \end{equation}
    Hence, problem (\ref{objective_norm_cut_sc}) has the same solution with problem 
    \begin{equation} \label{objective_relaxed_k_means}
        \max {\rm tr}(P^T Z Z^T P).
    \end{equation}
    On the other hand, if problem (\ref{objective_relaxed_k_means}) is equivalent  to problem (\ref{objective_norm_cut_sc}), $D = k I_n$. Therefore, we have
    \begin{equation}
        H Z^T Z \textbf{1}_n = H Z^T \cdot n \bm \mu = 0
    \end{equation}
    Accordingly, ${\rm rank}(H^T Z) < d$ which means that data points lie in an affine space, and the centralized data is perpendicular to $\bm \mu.$
    
    Hence, the theorem is proved.
\end{proof}

\begin{figure*}[t]
    \subfigure[Raw features]{
        \includegraphics[width=.3\linewidth]{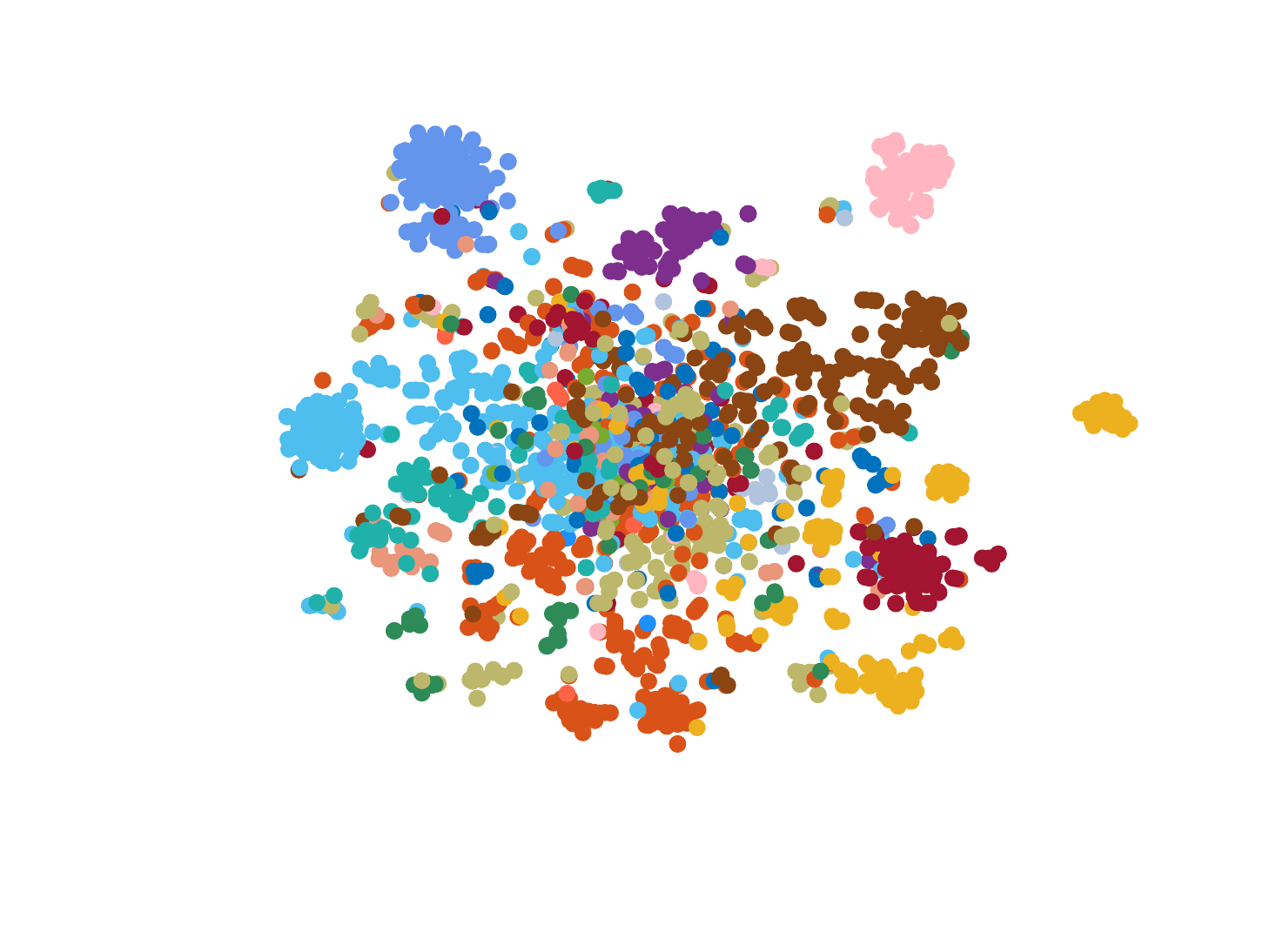}
    }
    \subfigure[Embedding of EGAE ($\alpha=0$)]{
        \includegraphics[width=.3\linewidth]{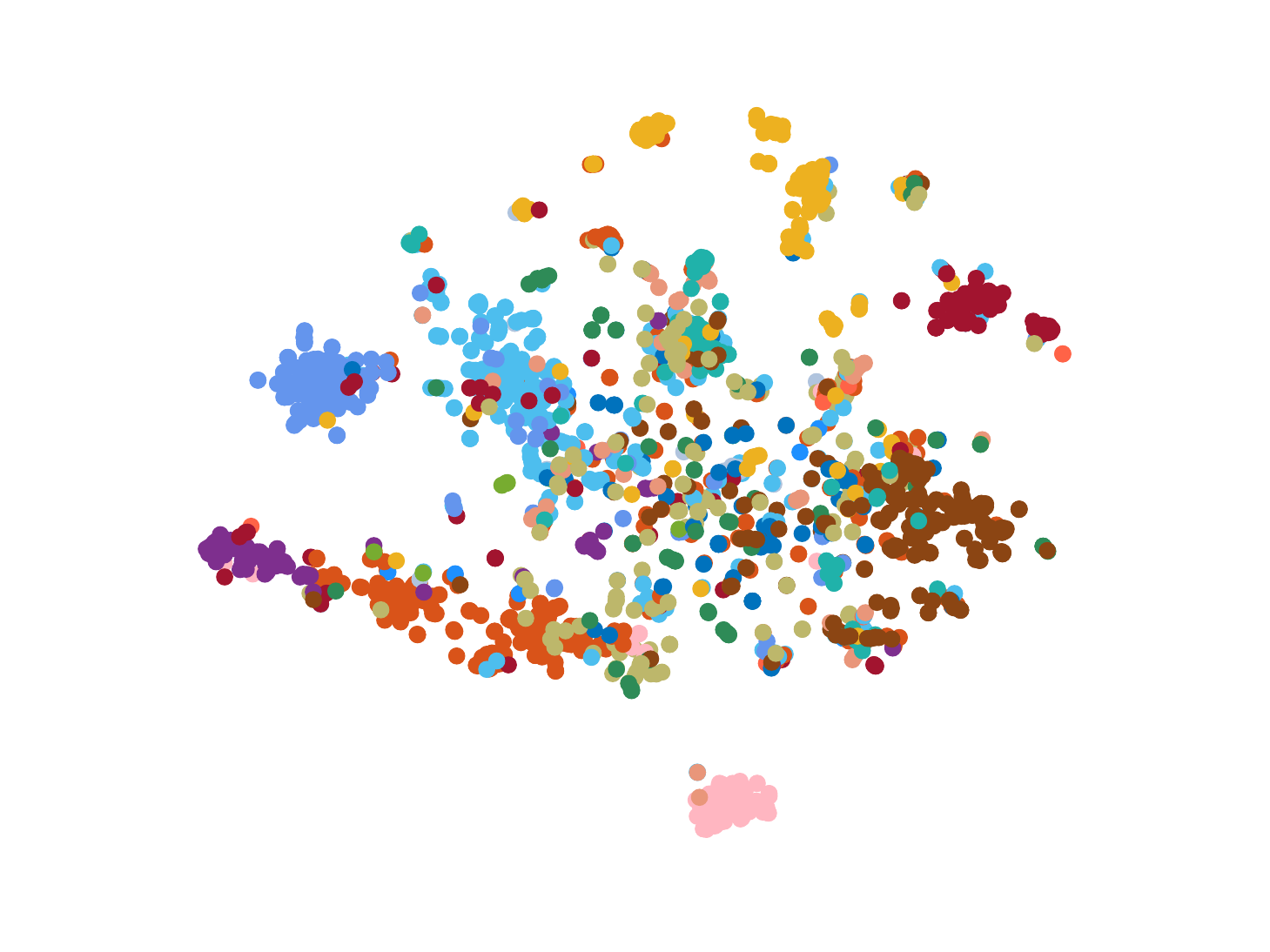}
    }
    \subfigure[Embedding of EGAE]{
        \includegraphics[width=.3\linewidth]{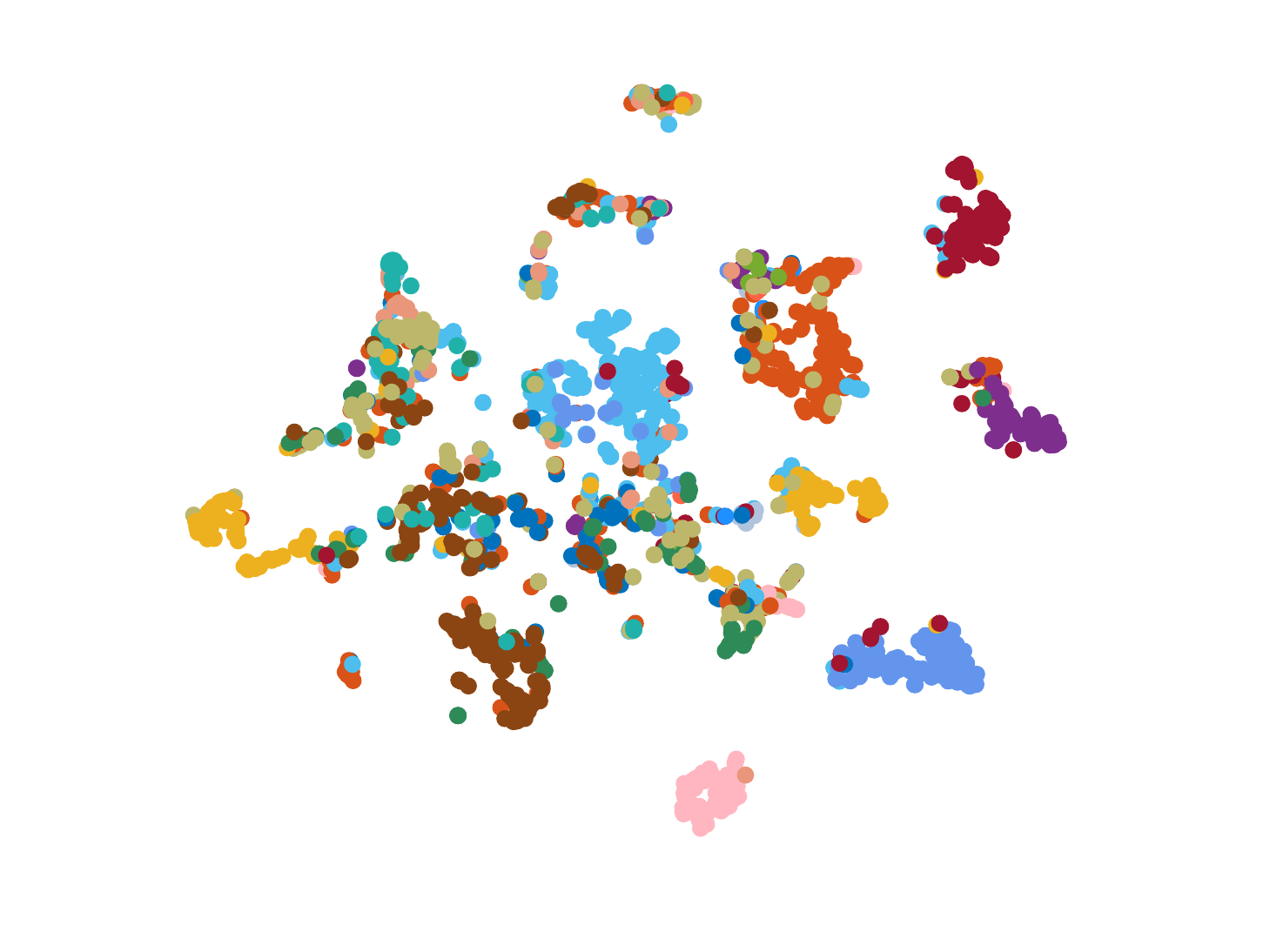}
    }
    \caption{t-SNE visualization of EGAE on Wiki.}
    \label{figure_illustration_wiki}
\end{figure*}

\subsection{Proof of Theorem \ref{theorem_eig}}

\begin{proof}
    Since $\|\bm z_i\|_2 = 1$, $(Q_z)_{ii} = 1$ always holds. 
    For each $Q^{(a)}$, we have 
    \begin{equation} \label{eq_trace_eig}
        \sum \limits_{i} \lambda^{(a)}_i = {\rm tr}(Q_z^{(a)}) = | \mathcal{C}_a |.
    \end{equation}
    Define 
    \begin{equation}
        \bm \gamma = \frac{1}{\sqrt{|\mathcal{C}_a|}} \textbf{1}_{|\mathcal{C}_a|}, 
    \end{equation}
    and we have 
    \begin{equation}
        \begin{split}
        \frac{\bm \gamma^T Q_z^{(a)} \bm \gamma}{\bm \gamma^T \bm \gamma} 
        = & \frac{1}{|\mathcal{C}_a|} (\frac{|\mathcal{C}_a|^2}{\epsilon} + |\mathcal{C}_a|(1 - \frac{1}{\epsilon})) \\
        = & \frac{|\mathcal{C}_a|}{\epsilon} + (1 - \frac{1}{\epsilon}).
        \end{split}
    \end{equation}
    The above equation indicates that 
    \begin{equation}
        \lambda_1^{(a)} \geq \frac{|\mathcal{C}_a|}{\epsilon} + (1 - \frac{1}{\epsilon}).
    \end{equation}
    Combine with Eq. (\ref{eq_trace_eig}), 
    \begin{equation}
        \lambda_2^{(a)} \leq (|\mathcal{C}_a| - 1) (1 - \frac{1}{\epsilon}).
    \end{equation}
    Here we get an upper bound for $\lambda_2^{(a)}$ and lower bound for $\lambda_1^{(a)}$. 
    To simplify the discussion, let 
    \begin{equation}
        \left \{
        \begin{array}{l}
            \mathcal{B}_{\epsilon}(|\mathcal{C}_a|) = \frac{|\mathcal{C}_a|}{\epsilon} + (1 - \frac{1}{\epsilon}) , \\
            \mathcal{U}_{\epsilon}(|\mathcal{C}_a|) = (|\mathcal{C}_a| - 1) (1 - \frac{1}{\epsilon}) .
        \end{array}
        \right .
    \end{equation}
    It should be pointed out that both $\mathcal{B}_{\epsilon}(|\mathcal{C}_a|)$
    and $\mathcal{U}_{\epsilon}(|\mathcal{C}_a|)$ increase with 
    $|\mathcal{C}_a|$ monotonously. 
    If the following inequality, 
    \begin{equation}
        1 \leq \epsilon < \frac{|\mathcal{C}_{min}|}{|\mathcal{C}_{max}| - 2} + 1 ,
    \end{equation}
    holds, then we have the following derivation,
    \begin{equation}
        \begin{split}
            & (\epsilon - 1) (|\mathcal{C}_{max}| - 2) < |\mathcal{C}_{min}| \\
            \Leftrightarrow ~ & (1 - \frac{1}{\epsilon}) (|\mathcal{C}_{max}| - 2) < \frac{|\mathcal{C}_{min}|}{\epsilon} \\
            \Leftrightarrow ~ & (1 - \frac{1}{\epsilon}) (|\mathcal{C}_{max}| - 1) < \frac{|\mathcal{C}_{min}|}{\epsilon} + (1 - \frac{1}{\epsilon}) \\
            \Leftrightarrow ~ & \mathcal{U}_{\epsilon}(|\mathcal{C}_{max}|) < \mathcal{B}_{\epsilon}(|\mathcal{C}_{min}|) .
        \end{split}
    \end{equation}
    Clearly, we have 
    \begin{equation}
        \begin{split}
            \lambda_1^{(a)} & \geq \mathcal{B}_{\epsilon}(|\mathcal{C}_a|) \geq \mathcal{B}_{\epsilon}(|\mathcal{C}_{min}|) \\
            & > \mathcal{U}_{\epsilon}(|\mathcal{C}_{max}|) \geq \mathcal{U}_{\epsilon}(|\mathcal{C}_{b}|) \geq \lambda_2^{(b)} .
        \end{split}
    \end{equation}
    Hence, the theorem is proved.
\end{proof}

\section{Experiment}
In this section, we elaborate on the crucial details of experiments 
and clustering results of EGAE. 
To be more intuitive, the t-SNE \cite{t-SNE} visualization is shown 
in Fig. \ref{figure_illustration_cora}, 
\ref{figure_illustration_citeseer}, and \ref{figure_illustration_wiki}. 

\begin{table}[t]
    \centering
    \caption{Concrete Information of Datasets}
    \label{table_datasets}
    \normalsize
    \begin{tabular}{l c c c c}
        \hline
        
        \hline
        Dataset & \# Nodes & \# Links & \# Features & \# Classes \\
        \hline
        \hline
        Cora & 2,708 & 5,429 & 1,433 & 7 \\
        Citeseer & 3,312 & 4,732 & 3,703 & 6 \\
        Wiki & 2,405 & 17,981 & 4,973 & 17 \\
        \hline

        \hline
    \end{tabular}
\end{table}

\begin{table*}[t]
    \centering
    \setlength{\tabcolsep}{4mm}
    \setlength{\belowcaptionskip}{1mm}
    \caption{Clustering Results (\%)}
    \label{table_graph_results}
    \normalsize
    \begin{tabular}{l c c c c c c c c c}
        \hline
        
        \hline
        \multirow{2}{*}{Methods} & \multicolumn{3}{c}{Cora} & \multicolumn{3}{c}{Citeseer} & \multicolumn{3}{c}{Wiki} \\
		 & ACC & NMI & ARI & ACC & NMI & ARI & ACC & NMI & ARI \\
        \hline
        \hline
        $k$-means & 49.18 & 32.05 & 22.81 & 53.97 & 30.50 & 27.81 & 40.43 & 42.91 & 15.03 \\
        SC \cite{sc2} & 36.72 & 12.67 & 3.11 & 23.89 & 5.57 & 1.00 & 22.04 & 18.17 & 1.46 \\
        Graph-Encoder \cite{graphencoder} & 32.49 & 10.93 & 0.55 & 22.52 & 3.30 & 1.00 & 20.67 & 12.07 & 0.49 \\
        DeepWalk \cite{deep_walk} & 48.40 & 32.70 & 24.27 & 33.65 & 8.78 & 9.22 & 38.46 & 32.38 & 17.03 \\
        DNGR \cite{dngr} & 41.91 & 31.84 & 14.22  & 32.59 & 18.02 & 4.29 & 37.58 & 35.85 & 17.97\\
        TADW \cite{TADW} & 56.03 & 44.11 & 33.20 & 45.48 & 29.14 & 22.81 & 30.96 & 27.13 & 4.54 \\
        \hline
        GAE \cite{gae} & 59.61 & 42.89 & 34.83 & 40.84 & 17.55 & 18.72 & 32.85 & 29.02 & 7.80 \\
        ARGE \cite{agae} & \underline{64.00} & 44.90 & 35.20 & \underline{57.30} & \underline{35.00} & 34.10 &  38.05 & 34.45 & 11.22 \\
        ARVGE \cite{agae} & 63.80 & \underline{45.00} & 37.40 & 54.40 & 26.10 & 24.50 & 38.67 & 33.88 & 10.69\\
        AGC \cite{AGC} & 68.92 & 53.68 & --- & 67.00 & 41.13 & --- & 47.65 & 45.28 & --- \\
        \hline
        \textbf{EGAE ($\alpha = 0$)} & 69.31 & 51.16  & 44.81 & 58.38 & 33.46 & 30.89 & 44.47 & 41.92 & 25.92 \\ 
        \textbf{EGAE} & \textbf{72.42} & \textbf{53.96} & \textbf{47.22} & \textbf{67.42} & \textbf{41.18} & \textbf{43.18} & \textbf{51.52} & \textbf{48.03} & \textbf{33.07} \\
        \hline

        \hline
    \end{tabular}
\end{table*}


\subsection{Benchmark Datasets} 
To verify the effectiveness of EGAE, experiments are conducted on 
3 graph datasets, including
\textit{Cora} \footnote{\url{linqs.soe.ucsc.edu/data} \label{url_cora}} \cite{cora}, 
\textit{Citeseer} \textsuperscript{\ref{url_cora}} \cite{cora}, 
and \textit{Wiki} \footnote{\url{github.com/thunlp/TADW}} \cite{TADW}.
These 3 datasets are citation networks 
where each node represents a publication 
and link denote a citation between two publications. 
Features of nodes are usually word vectors which may consist of keywords 
of papers (or pages).
Cora, which contains 7 categories of publications, has 2708 nodes and 5429 links. 
Citeseer, which contains 6 categories of publications, has 3312 nodes and 4732 links. 
Wiki has 2405 nodes and 4973 links while all nodes come from 17 classes.

\subsection{Baseline Methods and Evaluation Metrics} 
Totally ten algorithms are compared in the experiments: 
\textit{$k$-means}, Spectral Clustering (\textit{SC}), 
\textit{Graph-Encoder} \cite{graphencoder}, 
\textit{Deep Walk} \cite{deep_walk}, \textit{DNGR} \cite{dngr}, 
\textit{TADW} \cite{TADW}, \textit{GAE} (\textit{GAE}) \cite{gae}, 
Adversarial Regularized Graph Autoencoder (\textit{ARGE}) \cite{agae},
Adversarial Regularized Variational Graph Autoencoder \textit{ARVGE} (\textit{ARVGE}) \cite{agae}, 
and Adaptive Graph Convolution (\textit{AGC}) \cite{AGC}.
The experimental results of Graph-Encoder, Deep Walk, DNGR, TADW, ARGE, ARVGE, and AGC 
are reported from original papers.
In this paper, totally 3 metrics are employed to testify the performance 
of various models, including the clustering accuracy (\textit{ACC}), 
normalized mutual information (\textit{NMI}), 
and adjusted rand index (\textit{ARI}).

\subsection{Experimental Settings} 
In our experiments, the encoder is a two-layer GCN. 
Both activation functions of the two layers are ReLU. 
To have a stable training process, we employ a LASSO regularization for 
parameters of neural networks, and the corresponding tradeoff coefficient 
is set as $10^{-3}$ for Cora and Citeseer and $10^{-4}$ for Wiki.
In EGAE, $\alpha$ is an important hyper-parameter that are searched from 
$\{10^{-2}, 10^{-1}, 10^{0}, 10^{1}, 10^{2}, 10^{3}, 10^{4}, 10^{5}\}$.
To avoid the biased indicator matrix $P$ misleading the training of neural 
networks, we employ EGAE with $\alpha = 0$ as the pretrain for EGAE with a 
non-zero $\alpha$. 
The learning rate of EGAE with $\alpha = 0$ is set as $10^-3$ and the 
maximum iteration for training is set as $200$.
After pretraining, the learning rate of fine tuning is set as $10^{-4}$. 
The maximum iteration to update $P$ is set as $30$ and the maximum inner 
iteration to update the neural network is set as $5$.
For Cora, Citeseer, and Wiki, the encoder is composed of 256-neurons 
hidden layer and 128-neurons embedding layer. 
Although the number of clusters in all datasets is not too large 
and the embedding dimension just needs to equal with $c$ in theory, 
too small embedding dimension may lead to the slow convergence and
difficulty to train EGAE. 
Accordingly, the embedding layer is set as a 128-neuron graph 
convolution layer for all datasets. 
To test the effectiveness of the dual \textit{decoders} EGAE, 
we also report the performance of EGAE with $\alpha = 0$, the pretrain model.
When $\alpha = 0$, the model only has a unique decoder like other GAE-based 
models. 
Since clustering results of most compared methods and EGAE depend on $k$-means,
all methods are performed 10 times and the means are reported.  
All codes are implemented by torch-1.3.1 on a Win 10 PC with an NVIDIA
GeForce GTX 1660 GPU.

\subsection{Experimental Results} 
The clustering results of Cora, Citeseer, and Wiki are summarized in Table \ref{table_graph_results}. 
The best results are highlighted by boldface. 
From this table, we get some conclusions as follows:
\begin{itemize}
    \item On all datasets, GAE and its extensions outperform other 
            network embedding models, especially two traditional clustering models 
            $k$-means and spectral clustering. 
            Specifically speaking, GAE increases ACC by more than 10 percent 
            compared with Deep Walk. 
    \item Similar to traditional auto-encoder, GAE suffers from overfitting. 
            ARGE and ARVGE absorb adversarial learning into GAE to address 
            this issue and thus promote the performance such that they achieve 
            better performance. 
    \item In particular, EGAE obtains remarkable results on all metrics with 
            the same depth of encoder. 
            Compared with GAE, EGAE is a highly theory-driven model where
            the extracted deep features are more appropriate for relaxed 
            $k$-means. 
            On Cora, EGAE improves ACC by 4\% and 
            ARI by 9\% compared to the second-best method. 
            On Wiki, EGAE improves ACC by 4\%, 
            NMI by 3\%, and ARI by 16\% compared to the second-best method. 
    \item The effect of simultaneously learning GAE and clustering model is 
            apparent, especially on Citeseer and Wiki. 
            For instance, the joint learning improves ACC, NMI, and ARI 
            by 9, 8, and 13 percentage on Citeseer compared with EGAE 
            with $\alpha = 0$.
            ACC, NMI, and ARI increase by 7, 7, and 8 percentage on Wiki.
\end{itemize}

\begin{figure}[t]
    \centering
    \subfigure[Cora]{
        \includegraphics[width=0.46\linewidth]{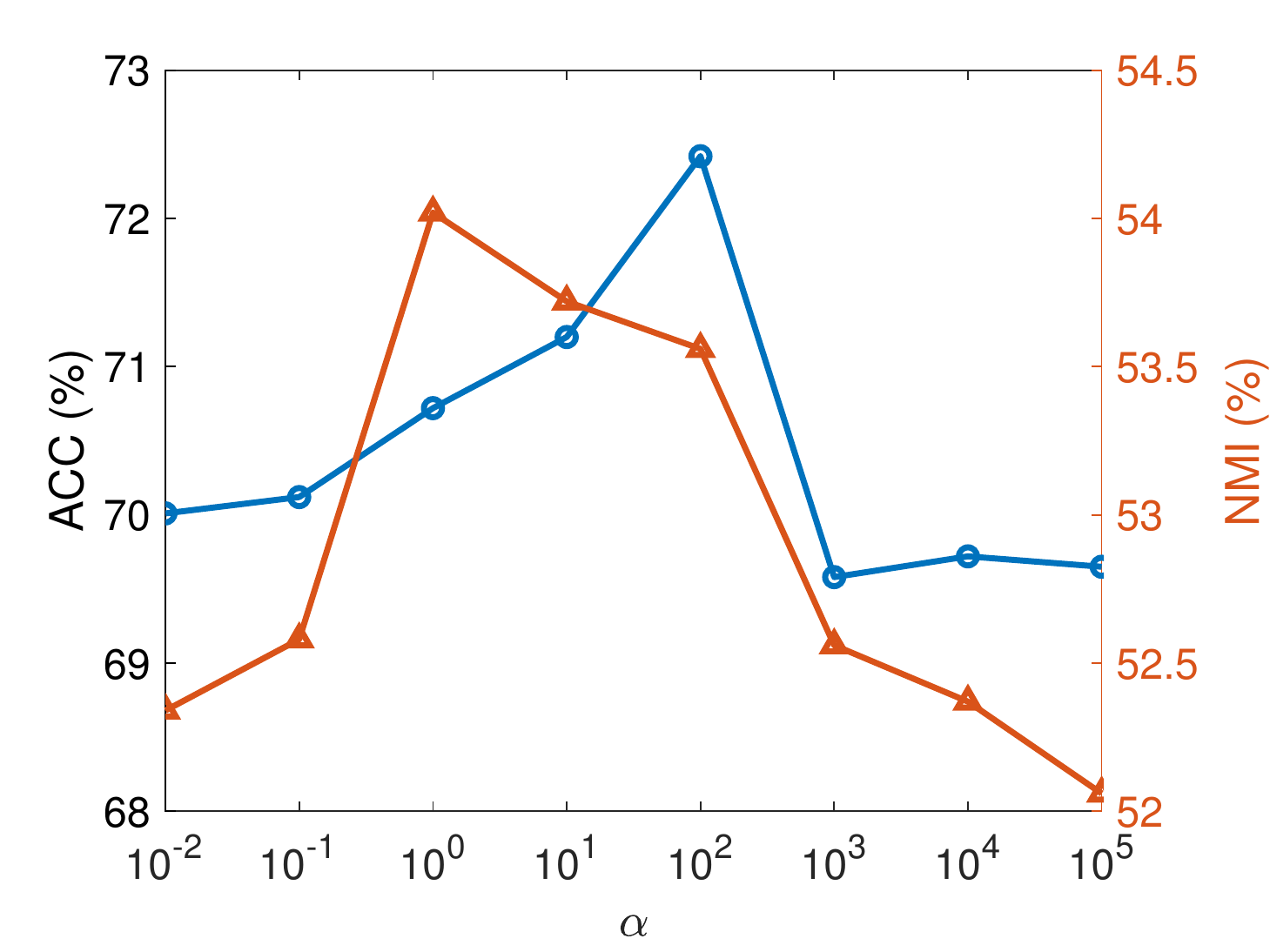}
    }
    \subfigure[Wiki]{
        \includegraphics[width=0.46\linewidth]{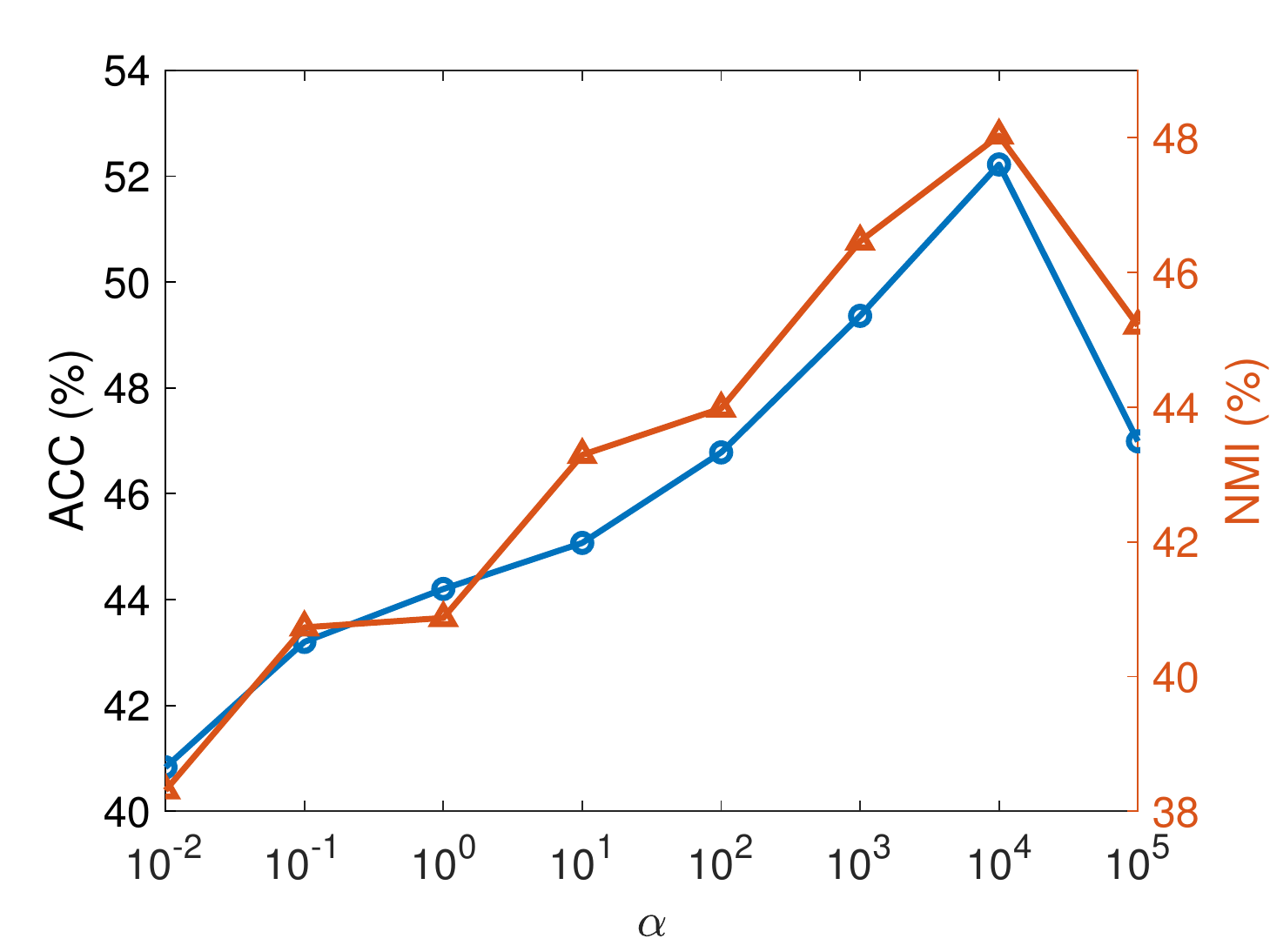}
    }
    \caption{Impact of $\alpha$ on Cora and Wiki. ACCs and NMIs of EGAE 
    with different $\alpha$ are shown.}
    \label{figure_dim_sensitivity}
\end{figure}

\subsection{Parameter Study} 
In our experiments, we focus on the impact of the tradeoff coefficient 
$\alpha$, which plays an important role in balancing two decoders. 
$\alpha$ varies in $\{10^{-2},$ $10^{-2},$ $10^{-1}, 10^0, 10^{1}, 10^{2}, 10^{3}, 10^{4}, 10^{5}\}$. 
Fig. \ref{figure_dim_sensitivity} demonstrates different performance with different $\alpha$ on Cora and Wiki. 
Not surprisingly, too large $\alpha$ may lead to generate biased embedding 
such that EGAE obtains undesirable results.
Combining the results shown in Table \ref{table_graph_results} and the two figures, 
we find that any individual decoder will not lead to the best results.

\section{Conclusion}

In this paper, we propose a novel GAE-based clustering model, 
Embedding Graph Auto-Encoder (\textit{EGAE}) for graph clustering,
driven by rigorous theoretical analyses about the relaxed $k$-means 
on inner-products space. 
Each component of EGAE is supported by theorems. 
We prove that Algorithm \ref{alg} will return an optimal solution 
in the ideal case.
Besides, experiments on synthetic datasets illustrate the theorems vividly.
Ablation experiments (EGAE and EGAE with $\alpha = 0$) also testify the 
effectiveness of the joint learning of embedding and clustering. 
Or equivalently, the results verify that collaborative training of the two 
decoders outperforms any single one. 
The analysis of computational complexity shows that the extra decoder 
will not result in expensive cost.
An extension that tries to rectify the biased indicator via sharing the 
adjacency matrix is further proposed. 
As the objective is a constrained problem, 
the optimization of EGAE is to update GAE by gradient descent and 
perform clustering alternatively. 
Extensive experiments prove the superiority of EGAE.

\bibliographystyle{IEEEbib}
\bibliography{EGAE-JOC}

\end{document}